\documentclass{article}

\usepackage{arxiv}

\usepackage[utf8]{inputenc}
\usepackage[T1]{fontenc}

\usepackage{graphicx}%
\usepackage{amsmath,amssymb,amsfonts}%
\usepackage{amsthm}%
\usepackage{mathrsfs}%
\usepackage[title]{appendix}%
\usepackage{xcolor}%
\usepackage{textcomp}%
\usepackage{hyperref}%
\usepackage{booktabs}%
\usepackage{algorithm}%
\usepackage{algorithmicx}%
\usepackage{algpseudocode}%
\usepackage{bm}%
\usepackage{tikz}%
\usetikzlibrary{positioning, arrows.meta, shapes.geometric, calc, decorations.pathreplacing}%

\theoremstyle{plain}%
\newtheorem{theorem}{Theorem}%
\newtheorem{proposition}[theorem]{Proposition}%

\theoremstyle{remark}%
\newtheorem{remark}{Remark}%

\theoremstyle{definition}%
\newtheorem{assumption}{Assumption}%

\def\WB{\mathrm{WB}}
\def\BB{\mathrm{BB}}
\newcommand{\bff}{{\bf f}}

\renewcommand{\shorttitle}{Exploiting Separability in Multi-Scale Grey-Box BO}

\title{Exploiting Separability in Multi-Scale Grey-Box Bayesian Optimization}

\author{Joshua E.\ Hammond$^{1}$ \and Tyler A.\ Soderstrom$^{2}$ \and Brian A.\ Korgel$^{1,3}$ \and Michael Baldea$^{1,4}$\thanks{Corresponding author: mbaldea@che.utexas.edu}}

\hypersetup{
pdftitle={Exploiting Separability in Multi-Scale Grey-Box Bayesian Optimization},
pdfauthor={Joshua E. Hammond, Tyler A. Soderstrom, Brian A. Korgel, Michael Baldea},
pdfkeywords={Bayesian optimization, Grey-box optimization, Bilevel optimization, Surrogate-based optimization, Multi-scale design, Dimensionality reduction},
}

\date{}

\begin{document}

\maketitle

\noindent$^{1}$McKetta Department of Chemical Engineering, The University of Texas at Austin, 200 E. Dean Keaton St. Stop C0400, Austin, Texas 78712, USA\\
$^{2}$ExxonMobil Technology and Engineering, 22777 Springwoods Village Pkwy, Spring, Texas 77389, USA\\
$^{3}$Energy Institute, The University of Texas at Austin, 2304 Whitis Ave.\ Stop C2400, Austin, Texas 78712, USA\\
$^{4}$Institute for Computational Engineering and Sciences, The University of Texas at Austin, 201 E.\ 24th Street, POB 4.102, Stop C0200, Austin, Texas 78712, USA

\bigskip

\begin{abstract}
We consider grey-box optimization problems where the decision variables naturally partition into black-box variables $x^{\BB}$ (as arguments to an expensive black-box function) and white-box variables $x^{\WB}$, governed by a set of explicit, closed-form equations that also depend on the output of the black-box function. We exploit this separability through a bilevel reformulation: an outer Bayesian optimization (BO) to optimize the scalar objective as a function of $x^{\BB}$ alone, while an inner problem solves the white-box subproblem via global optimization. The Gaussian process surrogate used in BO is therefore defined $\mathbb{R}^{n_{\BB}}$ rather than $\mathbb{R}^{n_{\WB}+n_{\BB}}$, and white-box constraints are satisfied exactly whenever the inner optimizer converges to a feasible point---without penalty functions, chance constraints, or moment approximations. On a suite of 13 benchmark problems, bilevel BO achieves lower regret, with fewer iterations and wall clock time. This advantage is robust to initialization set size, exploration parameters, and inner-solver choice.
\end{abstract}


\section{Introduction}\label{sec:introduction}

Surrogate-based optimization is the standard approach for expensive black-box functions~\cite{frazier2018tutorial}, yet many problems contain known, differentiable substructure that monolithic surrogates waste samples learning. When a model can be readily separated into an ``expensive'' black-box component and a ``cheap'' white-box component, building a single surrogate over all variables forces the surrogate model to learn both the unknown and the known components. The cost of this redundancy grows with dimension: as the size of the combined variable space $[x^{\WB}, x^{\BB}]$ expands, the curse of dimensionality burdens the surrogate for structure it should never have had to discover.

This situation arises naturally in multi-scale engineering problems where design and operation are jointly optimized. For example, a catalyst designer uses density functional theory (DFT) simulations to predict kinetic parameters, which are then used in well-understood reactor mass and energy balances to evaluate reaction yield. Similarly, a polymer engineer uses molecular simulations to estimate membrane transport properties, then solves the solution-diffusion equations to size a separation unit. In each case, the overall objective depends on two groups of decision variables. \emph{Black-box variables} $x^{\BB}$ are inputs to expensive simulations or experiments. Conversely, \emph{white-box variables} $x^{\WB}$ are variables present in the closed-form macroscopic equation. The black-box and white-box elements of the model are \emph{separable}: the expensive computation depends only on $x^{\BB}$, producing outputs $y = f^{\BB}(x^{\BB})$ that parameterize the white-box model. The key structural property is that, for any fixed $x^{\BB}$ and $y$, the remaining optimization over $x^{\WB}$ is a standard NLP solvable by conventional methods assuming that the $\WB$ equations are continuous and that $x^{\WB}$ are all real.

Several lines of work exploit known structure in expensive optimization, ranging from deterministic grey-box solvers~\cite{boukouvala2017greybox_jgo,boukouvala2017argonaut,eason2016trf} to composite and structured Bayesian optimization~\cite{astudillo2019bocf,paulson2022cobalt,kudva2025bonsai} to data-driven bilevel methods~\cite{beykal2020domino,kieffer2017bilevel}. We review these in Section~\ref{sec:related_work}. The common gap is that no prior method combines a surrogate-based global search over only the black-box variables with exact NLP solution of the white-box subproblem---a configuration that fully exploits variable separability for both dimensionality reduction and exact constraint satisfaction.

We exploit this separability through a bilevel reformulation: the outer loop uses Bayesian optimization (BO) to search over $x^{\BB}$ alone, while an inner optimizer solves the white-box subproblem exactly for each candidate. The Gaussian process (GP) surrogate therefore operates over $\mathbb{R}^{n_{\BB}}$ rather than $\mathbb{R}^{n_{\WB}+n_{\BB}}$, and white-box constraints are satisfied exactly whenever the inner optimizer converges to a feasible point. The details of this reformulation are presented in Section~\ref{sec:method}.

We make three contributions:
\begin{enumerate}
    \item A bilevel reformulation of separable grey-box problems that reduces surrogate dimensionality from $n_{\WB}+n_{\BB}$ to $n_{\BB}$ by solving the white-box subproblem via global optimization, with white-box constraints satisfied exactly whenever the inner optimizer converges to a feasible point---without penalty functions, chance constraints, or moment approximations.
    \item A benchmark suite of 13 separable grey-box problems (2--5 total variables, 0--3 constraints) considering synthetic test functions and prototype engineering applications---the largest such suite for this problem class.
    \item Comprehensive empirical evidence from 8{,}450 independent optimization runs (8{,}190 BO + 130 NLP + 130 BH): the proposed bilevel strategy achieves 11$\times$--$10^{8}\times$ lower regret than black-box BO across all 13 problems, with equal or faster wall time and robustness to hyperparameters ($n_\text{init}$, $\xi$) and inner-solver choice.
\end{enumerate}

%% ============================================================================
%% SECTION 2: PROBLEM FORMULATION (was Section 2)
%% ============================================================================
\section{Problem Formulation}\label{sec:problem_formulation}

We consider optimization problems in which an expensive black-box model is coupled with known, differentiable equations. This setting arises in multi-scale engineering design---molecular simulations feeding into process models---but also in any domain where part of the system is analytically tractable and part is not. We first formalize the general problem, then identify the structural property our method exploits.

\subsection{Problem Setting and Notation}\label{subsec:problem_setting}

The decision variables comprise two groups:
\begin{itemize}
    \item $x^{\WB} \in \mathbb{R}^{n_{\WB}}$: \emph{white-box variables} governed by known, differentiable equations (e.g., temperatures, pressures, flow rates in engineering; policy parameters in control; design dimensions in structural optimization). These variables enter the white-box model $f^{\WB}$, the objective $J$, and the inequality constraints $g$.
    \item $x^{\BB} \in \mathbb{R}^{n_{\BB}}$: \emph{black-box variables} that parameterize an expensive, non-differentiable function $f^{\BB}$ (e.g., molecular descriptors requiring DFT, material properties from simulations, hyperparameters of a costly simulator). These variables enter $f^{\BB}$ directly and may also appear in $J$ and $g$.
\end{itemize}

The black-box model
\begin{equation}\label{eq:molecular_model}
    y = f^{\BB}(x^{\BB}), \quad f^{\BB}: \mathbb{R}^{n_{\BB}} \to \mathbb{R}^{n_y},
\end{equation}
maps $x^{\BB}$ to a vector of intermediate parameters $y \in \mathbb{R}^{n_y}$ that couple the black-box and white-box components. We assume $f^{\BB}$ is expensive to evaluate and that we lack closed-form expressions or derivative information for it. Note that $y$ refers exclusively to the outputs of the black-box function; the white-box variables $x^{\WB}$ are not ``outputs'' in this sense but rather decision variables whose optimal values are determined by solving the white-box subproblem for a given $y$.

The white-box model takes the implicit form
\begin{equation}\label{eq:process_model}
    f^{\WB}(x^{\WB}, y) = 0, \quad f^{\WB}: \mathbb{R}^{n_{\WB}} \times \mathbb{R}^{n_y} \to \mathbb{R}^{n_{\WB}},
\end{equation}
encoding known equations (e.g., mass and energy balances) that relate the white-box variables $x^{\WB}$ to the black-box outputs $y$. Although $y$ may appear as a parameter in the white-box equations, it is fixed for any given $x^{\BB}$; the white-box optimization searches only over $x^{\WB}$. We assume $f^{\WB}$ is sufficiently smooth and that the Jacobian $\partial f^{\WB}/\partial x^{\WB}$ is available analytically.

The inequality constraints
\begin{equation}\label{eq:g_definition}
    g(x^{\WB}, y, x^{\BB}) \leq 0, \quad g: \mathbb{R}^{n_{\WB}} \times \mathbb{R}^{n_y} \times \mathbb{R}^{n_{\BB}} \to \mathbb{R}^{n_g},
\end{equation}
encode $n_g$ design specifications, safety limits, or feasibility requirements that may depend on both variable groups and the black-box outputs.

The complete optimization problem over both variable groups is
\begin{subequations}\label{eq:full_problem}
\begin{align}
    \min\limits_{x^{\WB}, x^{\BB}} \quad & J(x^{\WB}, y, x^{\BB}) \label{eq:objective}\\
    \text{s.t.} \quad & f^{\WB}(x^{\WB}, y) = 0 \label{eq:process_constraint}\\
    & y = f^{\BB}(x^{\BB}) \label{eq:material_constraint}\\
    & g(x^{\WB}, y, x^{\BB}) \leq 0 \label{eq:inequality_constraints}\\
    & x^{\WB} \in \mathcal{X}^{\WB}, \quad x^{\BB} \in \mathcal{X}^{\BB} \label{eq:bounds}
\end{align}
\end{subequations}

The objective $J(x^{\WB},\allowbreak y,\allowbreak x^{\BB})$ and inequality constraints $g(x^{\WB},\allowbreak y,\allowbreak x^{\BB}) \leq 0$ may depend on the white-box variables, the black-box outputs, and---in some problems---directly on the black-box variables $x^{\BB}$.

%% ============================================================================
%% SECTION 3: RELATED WORK
%% ============================================================================
\section{Related Work}\label{sec:related_work}

Problem~\eqref{eq:full_problem} sits at the intersection of surrogate-based global optimization, grey-box Bayesian optimization, bilevel programming, and constrained expensive optimization. We organize prior work along two axes that follow directly from Assumption~\ref{ass:separability}: (i) whether the method distinguishes $x^{\BB}$ from $x^{\WB}$, and (ii) whether the white-box subsystem $f^{\WB}$ and constraints $g$ are exploited exactly or approximated. Methods that collapse along axis (i) pay the curse of dimensionality in $n_{\WB}+n_{\BB}$; methods that collapse along axis (ii) trade exactness for the generality of uncertainty propagation.

\subsection{Surrogate-Based Global Optimization}

The use of surrogate models to guide expensive optimization traces to the Efficient Global Optimization (EGO) algorithm of Jones et al.~\cite{jones1998efficient}, which introduced Expected Improvement (EI) as an acquisition function for GP surrogates. EGO treats the objective as a monolithic black box and builds a single GP over the full decision space $\mathcal{X}^{\WB}\times\mathcal{X}^{\BB}\subset\mathbb{R}^{n_{\WB}+n_{\BB}}$ which is the default in most BO implementations~\cite{shahriari2016bo_survey,frazier2018tutorial}.

Within the deterministic surrogate realm, Boukouvala, Hasan, and Floudas~\cite{boukouvala2017greybox_jgo} developed a methodology for constrained grey-box problems that selects surrogates from a diverse set of functional forms and globally optimizes the resulting NLP. The companion ARGONAUT framework~\cite{boukouvala2017argonaut} formalized this into an iterative algorithm with variable selection, bounds tightening, and constrained sampling for problems with up to 100 variables; Kieslich et al.~\cite{kieslich2018smolyak} extended the approach using Smolyak grids and polynomial surrogates. These methods share our philosophy of exploiting known equations, but they build surrogates over the joint $(x^{\WB}, x^{\BB})$ space rather than restricting the surrogate to $x^{\BB}$ as Assumption~\ref{ass:separability} permits; $f^{\WB}$ is used only to generate training data or to tighten bounds on $\mathcal{X}^{\WB}$, not to eliminate $x^{\WB}$ from the surrogate's domain. Moreover, the deterministic surrogates they employ lack the principled exploration--exploitation trade-off of probabilistic models and tend to overexploit early samples under tight evaluation budgets.

The trust-region filter algorithms of Eason and Biegler~\cite{eason2016trf,eason2018trf} offer a complementary local approach: they embed surrogates for $f^{\BB}$ within a trust region and solve the white-box NLP in $x^{\WB}$ at each iteration---architecturally close to the inner subproblem in Assumption~\ref{ass:separability}(ii), but with only local convergence guarantees in $x^{\BB}$ rather than a global exploration mechanism.

\subsection{Grey-Box Bayesian Optimization}

Grey-box BO methods incorporate known structure into the surrogate to reduce its modeling burden. Astudillo and Frazier~\cite{astudillo2019bocf} introduced Bayesian Optimization of Composite Functions (BOCF), which in our notation models $J(x^{\WB}, f^{\BB}(x^{\BB}), x^{\BB})$ by training a multi-output GP on $f^{\BB}$ and propagating uncertainty through $J$ via sampling; later extensions handle function networks~\cite{astudillo2021bofn} and partial evaluations~\cite{buathong2024partial}. Paulson and Lu~\cite{paulson2022cobalt} extend this to constrained grey-box problems by combining multivariate GP models with a constrained expected utility function, using sample average approximation and chance constraints to propagate GP uncertainty through known equations. Gonz\'{a}lez and Zavala~\cite{gonzalez2024bois} propose an interconnected-systems BO framework with adaptive linearization, Winz et al.~\cite{winz2025greybox} modify the upper confidence bound to exploit grey-box derivative information, and Kudva and Paulson~\cite{kudva2025bonsai} exploit function-network topology for robust optimization.

All of these methods place a surrogate on the intermediate outputs $y$ defined in Equation~\eqref{eq:molecular_model} and propagate uncertainty through Equation~\eqref{eq:process_constraint}, requiring multi-output GPs and moment approximations or sampling. When Assumption~\ref{ass:separability}(ii) holds---i.e., the white-box subproblem is solvable to global optimality---using a surrogate model for $y$ adds modeling complexity without clear benefit. We instead place a scalar GP on the value function
\begin{equation}\label{eq:value_function}
    x^{\BB} \;\mapsto\; J\bigl(x^{\WB\star}(x^{\BB}),\; f^{\BB}(x^{\BB}),\; x^{\BB}\bigr)
\end{equation}
in $\mathbb{R}^{n_{\BB}}$, avoiding both multi-output surrogates and uncertainty propagation (Table~\ref{tab:method_comparison}).

\subsection{Bilevel Optimization}

Bilevel programming has a rich history in the global optimization literature. G\"{u}m\"{u}\c{s} and Floudas~\cite{gumus2001bilevel} developed the first rigorous global optimization approach for nonlinear bilevel problems using the $\alpha$BB framework. Mitsos, Chachuat, and Barton~\cite{mitsos2009bilevel_dynamic} extended this to bilevel dynamic optimization, Kleniati and Adjiman~\cite{kleniati2014bilevel} addressed the mixed-integer case, and Fa{\'\i}sca et al.~\cite{faisca2007bilevel} proposed parametric global optimization for bilevel programs via multi-parametric programming. In the language of Section~\ref{sec:problem_formulation}, these methods treat $x^{\BB}$ as the upper-level decision and $x^{\WB}$ as the lower-level decision, but assume analytical access to both $f^{\BB}$ and $f^{\WB}$---an assumption that breaks when $f^{\BB}$ is an expensive black-box simulation.

Data-driven approaches address this gap. The DOMINO framework of Beykal et al.~\cite{beykal2020domino} solves bilevel mixed-integer nonlinear problems by sampling the upper-level objective and solving the lower-level problem to global optimality at each sample point, using a grey-box optimization solver (ARGONAUT/NOMAD) for the upper-level search. DOMINO is the closest non-BO predecessor to our approach: it solves the $x^{\WB}$ subproblem to global optimality at each sampled $x^{\BB}$, matching Assumption~\ref{ass:separability}(ii), but uses deterministic surrogates over $x^{\BB}$. We instead use a GP with a principled acquisition function that balances exploration and exploitation---an advantage when the evaluation budget is small ($<$250 calls) and the risk of overexploiting early samples is high.

Within the BO literature, Kieffer et al.~\cite{kieffer2017bilevel} first applied BO to bilevel problems, using EI at the outer level and Sequential Least Squares Programming (SLSQP) at the inner level. This nested architecture matches our outer-over-$x^{\BB}$/inner-over-$x^{\WB}$ decomposition, but Kieffer et al.\ did not exploit the separability of Assumption~\ref{ass:separability} to reduce surrogate dimensionality from $n_{\WB}+n_{\BB}$ to $n_{\BB}$. Recent work addresses the harder setting where \emph{both} levels are black boxes: Ekmekcioglu et al.~\cite{ekmekcioglu2024bilevel} model both levels as GPs over the joint space; Chew et al.~\cite{chew2025bilbo} construct confidence-bound trusted sets with regret guarantees; Fu et al.~\cite{fu2024convergence} provide convergence analysis for BO bilevel problems where the inner loop is unconstrained Stochastic Gradient Descent (SGD). When the inner problem is a known white-box NLP, solving it exactly is both cheaper and more accurate than learning a second surrogate.

Baldea~\cite{baldea2026multiscale} introduced the multi-scale bilevel BO framework we extend here, demonstrating it on a single chemical process case study. We advance that work with a 13-problem benchmark suite, with systematic hyperparameter sweeps, and detailed comparisons quantifying when and why bilevel decomposition helps.

\subsection{Constraint Handling in Expensive Optimization}

Standard constrained BO methods treat each component of $g(x^{\WB},\allowbreak y,\allowbreak x^{\BB})$ as an independent expensive black box with its own GP. Gardner et al.~\cite{gardner2014} introduced weighted expected improvement for unknown constraints; Gelbart et al.~\cite{gelbart2014unknown} extended this with probability-of-feasibility weighting; Picheny et al.~\cite{picheny2016albo} proposed a slack-variable augmented Lagrangian approach. These methods are appropriate when constraint functions are truly unknown, but when components of $g$ depend on $x^{\WB}$ and $y$ through known equations---mass balances, safety envelopes, thermodynamic limits---creating surrogates increases computational cost and introduces approximation error.

Our inner global optimizer enforces the white-box components of $g$ exactly via Assumption~\ref{ass:separability}(ii), avoiding penalty functions, chance constraints, and moment approximations. The advantage is largest when feasible regions are tight: in such cases(reflected in the Distillation and Heat-Exchanger problems in Section~\ref{sec:experiments}), bilevel BO achieves $>$100{,}000$\times$ lower regret than penalty-based black-box BO.

%% ============================================================================
%% SECTION 4: METHOD
%% ============================================================================
\section{Method: Bilevel Bayesian Optimization}\label{sec:method}

Problem~\eqref{eq:full_problem} possesses a structural property that we formalize as an assumption and then exploit algorithmically.

\begin{assumption}[Separable grey-box structure]\label{ass:separability}
Problem~\eqref{eq:full_problem} satisfies:
\begin{enumerate}
    \item[(i)] The black-box function $f^{\BB}$ depends only on $x^{\BB}$. For any fixed $x^{\BB}$ and $y = f^{\BB}(x^{\BB})$, the resulting optimization over $x^{\WB}$ alone is a well-posed NLP. %(The objective $J$ and constraints $g$ may depend on $x^{\BB}$ directly, not only through $y$; the separability is in the \emph{optimization structure}---the inner problem optimizes $x^{\WB}$ while $x^{\BB}$ and $y$ are held fixed.)
    \item[(ii)] For any fixed $x^{\BB}$ and $y = f^{\BB}(x^{\BB})$, the white-box subproblem
    \begin{subequations}\label{eq:inner_subproblem}
    \begin{align}
        \min\limits_{x^{\WB}} \; & J(x^{\WB},\allowbreak y,\allowbreak x^{\BB}) \label{eq:inner_sub_obj}\\
        \text{s.t.} \;\; & f^{\WB}(x^{\WB}, y) = 0 \label{eq:inner_sub_eq}\\
        & g(x^{\WB},\allowbreak y,\allowbreak x^{\BB}) \leq 0 \label{eq:inner_sub_ineq}\\
        & x^{\WB} \in \mathcal{X}^{\WB} \label{eq:inner_sub_bounds}
    \end{align}
    \end{subequations}
    is solvable to global optimality (of the subproblem) by a suitable global optimizer.
\end{enumerate}
\end{assumption}

Under Assumption~\ref{ass:separability}, the optimization can be decomposed as follows: for any fixed $x^{\BB}$, the black-box evaluation $y = f^{\BB}(x^{\BB})$ and the inner optimization over $x^{\WB}$ are self-contained. This decoupling is the structural feature we exploit: rather than optimizing all variables jointly, we can search over the low-dimensional black-box space while solving the white-box subproblem exactly at each step. Note that $f^{\BB}$ can always be augmented with identity mappings (e.g., $y_{n_y+1} = x^{\BB}_1$) to absorb any direct $x^{\BB}$ dependence in $J$ or $g$, so the distinction between ``depends on $x^{\BB}$ through $y$'' and ``depends on $x^{\BB}$ directly'' is a modeling choice, not a structural limitation.

\subsection{Bilevel Reformulation}\label{subsec:bilevel_reformulation}

Based on Assumption~\ref{ass:separability}, we reformulate~\eqref{eq:full_problem} as a bi-level optimization problem that separates what we know from what we must learn:
\begin{subequations}\label{eq:bilevel}
\begin{align}
    \min\limits_{x^{\BB}} \quad & J(x^{\WB\star}, y, x^{\BB}) \label{eq:outer_obj}\\
    \text{s.t.} \quad & y = f^{\BB}(x^{\BB}) \label{eq:outer_blackbox}\\
    & x^{\WB\star} = \arg\min\limits_{x^{\WB}} \; J(x^{\WB}, y, x^{\BB}) \label{eq:inner_obj}\\
    & \quad\quad \text{s.t.} \quad f^{\WB}(x^{\WB}, y) = 0 \label{eq:inner_process}\\
    & \quad\quad\quad\quad\; g(x^{\WB}, y, x^{\BB}) \leq 0 \label{eq:inner_ineq}\\
    & \quad\quad\quad\quad\; x^{\WB} \in \mathcal{X}^{\WB} \label{eq:inner_bounds}\\
    & x^{\BB} \in \mathcal{X}^{\BB} \label{eq:outer_bounds}
\end{align}
\end{subequations}

The \textbf{inner problem}~\eqref{eq:inner_obj}--\eqref{eq:inner_bounds} is a standard nonlinear program (NLP). Given fixed $x^{\BB}$ and $y = f^{\BB}(x^{\BB})$, it optimizes only the white-box variables $x^{\WB}$ using the white box model; $x^{\BB}$ and $y$ enter as parameters. Because the white-box equations are closed-form and inexpensive to evaluate, this subproblem can be solved to global optimality using a global optimizer such as Basin-Hopping (BH) with an SLSQP local minimizer.

The \textbf{outer problem}~\eqref{eq:outer_obj}--\eqref{eq:outer_blackbox} searches over black-box decisions $x^{\BB}$. For each candidate $x^{\BB}$, we evaluate the black box to obtain $y$, solve the inner problem to global optimality to get $x^{\WB\star}$, and record the resulting objective value $J(x^{\WB\star},\allowbreak y,\allowbreak x^{\BB})$. This outer loop is where we apply Bayesian optimization. In practice, the number of outer-loop iterations---each requiring one evaluation of $f^{\BB}$---is the primary driver of total campaign time, since each black-box evaluation may involve hours of computation or physical experimentation. The inner NLP solve adds negligible overhead by comparison (see Remark~4 and Section~\ref{sec:discussion}).

\begin{remark}
The key advantage is dimensionality reduction. The BO surrogate now operates over $\mathbb{R}^{n_{BB}}$ rather than $\mathbb{R}^{n_{WB} + n_{BB}}$. When $n_{BB} \ll n_{WB}$, the number of black-box evaluations required to build an accurate surrogate drops substantially relative to applying BO to problem~(\ref{eq:full_problem}).
\end{remark}

\begin{remark}
If the inner problem~\eqref{eq:inner_obj}--\eqref{eq:inner_bounds} is infeasible for some $y$, we return a large penalty value. The BO acquisition function will naturally steer away from black-box variable choices that lead to infeasible white-box solutions.
\end{remark}

\begin{remark}
    The bilevel decomposition remains valid because, for any fixed $x^{\BB}$ and $y = f^{\BB}(x^{\BB})$, the inner optimization is over $x^{\WB}$ alone. The sets $\mathcal{X}^{\WB}$ and $\mathcal{X}^{\BB}$ define box constraints.
\end{remark}

\begin{remark}[White-box cost]
    Calling the white-box subproblem ``inexpensive'' is relative to the black-box evaluation. In the simplest case, $f^{\WB}$ reduces to closed-form function evaluations (as in the benchmark problems of Section~\ref{sec:benchmark_suite}). More generally, the white-box model may require solving systems of nonlinear equations (e.g., iterative mass and energy balances, phase equilibrium calculations). The bilevel decomposition remains beneficial whenever the cost of solving the inner NLP is substantially less than the cost of evaluating $f^{\BB}$---a condition that holds broadly in multi-scale problems where each black-box evaluation involves hours of simulation or laboratory experimentation.
\end{remark}

% \subsection{Algorithm}\label{subsec:algorithm}
\subsection{Algorithm, Surrogate Model, and Acquisition Function}\label{subsec:algorithm}

\begin{algorithm}[htbp]
\caption{Bilevel Bayesian Optimization}\label{alg:bilevel_bo}
\begin{algorithmic}[1]
\Require Black-box bounds $\mathcal{X}^{\BB}$, white-box bounds $\mathcal{X}^{\WB}$, models $f^{\BB}$, $f^{\WB}$, $J$, $g$; initial samples $n_\text{init}$, iterations $N$, exploration parameter $\xi$
\Ensure Best $x^{\BB\star}$, $x^{\WB\star}$, $J^\star$
\Statex \textit{// Initialization}
\State Sample $\{x^{\BB}_i\}_{i=1}^{n_\text{init}}$ via Latin hypercube in $\mathcal{X}^{\BB}$
\For{$i = 1, \ldots, n_\text{init}$}
    \State Evaluate $y_i = f^{\BB}(x^{\BB}_i)$
    \State Solve inner problem (BH): $x^{\WB\star}_i = \arg\min_{x^{\WB}} J(x^{\WB}, y_i, x^{\BB}_i)$ s.t.\ $f^{\WB}(x^{\WB}, y_i)=0$, $g(x^{\WB}, y_i, x^{\BB}_i) \leq 0$
    \State Record $J_i = J(x^{\WB\star}_i)$ (or penalty if infeasible)
\EndFor
\Statex \textit{// BO loop}
\For{$t = n_\text{init}+1, \ldots, n_\text{init}+N$}
    \State Fit GP surrogate $\mathcal{GP}$ to $\{(x^{\BB}_i, J_i)\}_{i=1}^{t-1}$
    \State Select $x^{\BB}_t = \arg\max_{x^{\BB} \in \mathcal{X}^{\BB}} \alpha_\text{EI}(x^{\BB}; \mathcal{GP}, \xi)$
    \State Evaluate $y_t = f^{\BB}(x^{\BB}_t)$
    \State Solve inner problem (BH) $\to x^{\WB\star}_t$, record $J_t$
    \State Update dataset
\EndFor
\State \Return best $(x^{\BB\star}, x^{\WB\star}, J^\star)$
\end{algorithmic}
\end{algorithm}

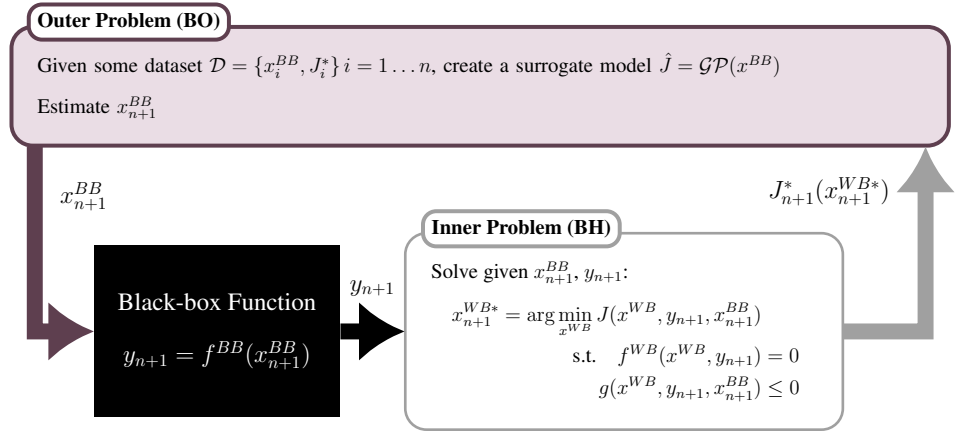
\begin{figure}[htbp]
    \centering
    % TikZ diagram inlined (template requires no \input{} for tex files)
    % Define colors from the image
    \definecolor{outerfill}{HTML}{EADDE5} % Light pink/purple for top box fill
    \definecolor{outerborder}{HTML}{5E3F4F} % Darker border for top box
    \definecolor{arrowpurple}{HTML}{5E3F4F} % Dark purple for the thick arrow
    \definecolor{innerborder}{HTML}{A0A0A0} % Light gray border for inner box

    \begin{tikzpicture}[
        scale=0.65, % Scale down the entire figure
        transform shape, % Scale text and line widths proportionally
        font=\sffamily,
        >=Latex,
        node distance=1.5cm,
        outer box/.style={
            draw=outerborder,
            fill=outerfill,
            line width=1.5pt,
            rounded corners=8pt,
            inner sep=15pt,
            align=left,
            font=\large,
            text width=18cm
        },
        title tab/.style={
            draw=outerborder,
            fill=white,
            line width=1.5pt,
            rounded corners=5pt,
            anchor=north west,
            xshift=10pt,
            yshift=13pt,
            font=\bfseries\large,
            inner sep=5pt
        },
        black box/.style={
            draw=none,
            fill=black,
            rectangle,
            minimum width=5cm,
            minimum height=3.5cm,
            text=white,
            align=center,
            font=\Large
        },
        inner box/.style={
            draw=innerborder,
            fill=white,
            line width=1pt,
            rounded corners=8pt,
            inner sep=15pt,
            align=left,
            font=\large,
            minimum width=6cm,
            minimum height=4cm
        },
        inner title tab/.style={
            title tab,
            draw=innerborder
        },
        arrow label/.style={
            font=\Large\bfseries
        }
    ]

    \node[outer box] (outer) at (0,0) {
        Given some dataset $\mathcal{D} = \{x_i^{BB}, J_i^*\} \, i = 1 \dots n$, create a surrogate model $\hat{J} = \mathcal{GP}(x^{BB})$ \\[0.8em]
        Estimate $x_{n+1}^{BB}$
    };
    \node[title tab] at (outer.north west) {Outer Problem (BO)};

    \node[black box, anchor=north west] (blackbox) at ([yshift=-2cm, xshift=1.7cm]outer.south west) {
        Black-box Function \\[1em]
        $y_{n+1} = f^{BB}(x_{n+1}^{BB})$
    };

    \node[inner box, right=1.3cm of blackbox] (inner) {
        Solve given $x_{n+1}^{BB}$, $y_{n+1}$: \\[0.8em]
        \quad $\begin{aligned}
            x_{n+1}^{WB*} = \arg \min_{x^{WB}} J(x^{WB}, y_{n+1}, x_{n+1}^{BB}) \\
            \text{s.t.} \quad f^{WB}(x^{WB}, y_{n+1}) &= 0 \\
            g(x^{WB}, y_{n+1}, x_{n+1}^{BB}) &\leq 0
        \end{aligned}$
    };
    \node[inner title tab] at (inner.north west) {Inner Problem (BH)};

    \draw[->, line width=6pt, outerborder, -{Latex[length=6mm, width=8mm]}]
        ([xshift=0.5cm]outer.south west) |- (blackbox.west);

    \node[arrow label, left=0.1cm] at ([xshift=0.5cm, yshift=1cm]blackbox.north west) {$x_{n+1}^{BB}$};

    \draw[->, line width=6pt, black, -{Latex[length=6mm, width=8mm]}]
        (blackbox.east) -- (inner.west);

    \node[arrow label, above=0.6cm] at ($(blackbox.east)!0.5!(inner.west)$) {$y_{n+1}$};

    % \draw[->, fill=white, bor]
    \draw[->, line width=6pt, white, draw=innerborder, -{Latex[length=6mm, width=8mm]}]
        (inner.east) -| ([xshift=-0.5cm]outer.south east);

    \node[arrow label, left=0.2cm] at ([xshift=-0.9 cm, yshift=-0.9cm]outer.south east) {$J_{n+1}^*(x_{n+1}^{WB*})$};

    % \begin{scope}[shift={(inner.east)}, xshift=1cm, yshift=-1.5cm, scale=1.2]
    %     \coordinate (base left) at (0,0);
    %     \coordinate (base right) at (1.5,0);
    %     \coordinate (shaft top left) at (0, 3);
    %     \coordinate (shaft top right) at (1.5, 3);
    %     \coordinate (head tip) at (0.75, 4.5);
    %     \coordinate (head left) at (-0.5, 3);
    %     \coordinate (head right) at (2.0, 3);

    %     \draw[draw=innerborder, line width=1.5pt, fill=white, rounded corners=4pt]
    %         (base left) -- (base right) --
    %         (shaft top right) -- (head right) --
    %         (head tip) --
    %         (head left) -- (shaft top left) -- cycle;

    %     \node[arrow label, below=0.2cm] at ($(base left)!0.5!(base right)$) {$J_{n+1}^*(x_{n+1}^{WB*})$};
    % \end{scope}

    \end{tikzpicture}
    \caption{Flow of the multi-scale Bayesian optimization framework. The outer loop uses BO to select black-box variables $x^{\BB}$, which are evaluated through the black-box model to produce parameters $y$. The inner loop solves the white-box optimization over $x^{\WB}$ via Basin-Hopping given $x^{\BB}$ and $y$, returning the optimal objective $J(x^{\WB\star},\allowbreak y,\allowbreak x^{\BB})$ to update the BO surrogate.}
    \label{fig:multiscale_diagram}
\end{figure}

% \subsection{Surrogate Model and Acquisition Function}\label{subsec:surrogate_and_acquisition}
We describe the proposed approach in Algorithm~\ref{alg:bilevel_bo} and Figure~\ref{fig:multiscale_diagram} and provide a discussion as follows: The GP surrogate in Algorithm~\ref{alg:bilevel_bo} models the scalar mapping~\eqref{eq:value_function}---a function of $x^{\BB}$ alone. In our implementation, we use a Gaussian process with an automatic relevance determination (ARD) radial basis function (RBF) kernel, where a separate length scale per dimension allows the GP to adapt to anisotropic landscapes. The outer loop selects the next evaluation by maximizing Expected Improvement (EI) with exploration parameter $\xi \geq 0$. Kernel and acquisition function details for our implementation are in Appendices~\ref{sec:gp_details} and~\ref{sec:acquisition_functions}; here we emphasize that both are standard---the gains reported in Section~\ref{sec:results} come from problem structure, not from novel surrogate or acquisition choices, although we expect that adopting them would yield further benefits.

\subsection{Key Properties}\label{subsec:key_properties}

The bilevel reformulation~\eqref{eq:bilevel} has several properties worth highlighting. The first is formalized as a proposition.

\begin{proposition}\label{prop:dimensionality}
Under Assumption~\ref{ass:separability}, the bilevel reformulation~\eqref{eq:bilevel} preserves the global optimum of~\eqref{eq:full_problem} and reduces the GP surrogate domain from $\mathbb{R}^{n_{\WB}+n_{\BB}}$ to $\mathbb{R}^{n_{\BB}}$.
\end{proposition}

\noindent The proof is in Appendix~\ref{sec:proof_dimensionality}. The key consequences are threefold:

\emph{Exact constraint satisfaction.} Constraints $g(x^{\WB},\allowbreak y,\allowbreak x^{\BB}) \leq 0$ are handled by the inner global optimizer, not by penalty functions or chance constraints. The GP never needs to learn the constraint boundary. When the inner optimizer converges to a globally feasible point (Assumption~\ref{ass:separability}(ii)), feasibility is guaranteed. In our implementation, Basin-Hopping (BH)~\cite{wales2003energy,wales1997global, li1987monte, wales1999global, olson2012basin} with Sequential Least Squares Programming (SLSQP)~\cite{nocedal2006numerical, kraft1988software, lawson1995solving} local minimization provides high-confidence global solutions. We also compare a simpler solution using SLSQP with multiple restarts as a more time-efficient method to solve the reduced-dimension inner problem. Further guarantees of global optimization could be obtained with deterministic global solvers (e.g., BARON), as discussed in Section~\ref{sec:discussion}. Details on these solvers and implementation are available in Appendix~\ref{sec:solver_benchmarks}.

\emph{Scalar surrogate.} The method requires only one GP for the scalar optimal objective $J(x^{\WB\star},\allowbreak y,\allowbreak x^{\BB})$. Approaches that surrogate the black-box outputs $y \in \mathbb{R}^{n_y}$ instead require a multi-output GP, with the associated challenges of cross-covariance specification and $O(n_y^3)$ scaling.

Additional implementation details---infeasibility handling and solver modularity---are discussed in Appendix~\ref{sec:solver_benchmarks}.

\subsection{Relationship to Existing Grey-Box Methods}\label{subsec:contrast_greybox_approaches}

% Comparison table: our method vs COBALT vs BOCF vs full-space BO
\begin{table}[htbp]
\centering
\caption{Comparison of grey-box BO approaches. Bilevel BO (this work) is the only method that both separates variables and solves the white-box subproblem exactly.}
\label{tab:method_comparison}
\resizebox{\textwidth}{!}{%
\begin{tabular}{@{}lcccc@{}}
\toprule
& \textbf{Bilevel BO} & \textbf{COBALT} & \textbf{BOCF} & \textbf{Full-space BO} \\
& (this work) & \cite{paulson2022cobalt} & \cite{astudillo2019bocf} & \\
\midrule
What is modeled & $x^{\BB} \mapsto J^\star$ & $x^{\BB} \mapsto y$ & $x \mapsto h(x)$ & $x \mapsto J$ \\
Model type & scalar GP & multi-output GP & vector GP & scalar GP \\
Model dimension & $n_{\BB}$ & $n_{\BB}$ & $n_{\WB}+n_{\BB}$ & $n_{\WB}+n_{\BB}$ \\
Variables separated? & Yes & Yes & No & No \\
Constraint handling & exact (global opt.) & chance constraints & penalty & penalty \\
Inner optimization & global optimization (BH) & moment propagation & none & none \\
$y$-dependent constraints & via inner optimizer & natively & penalty & penalty \\
\bottomrule
\end{tabular}}
\end{table}

Table~\ref{tab:method_comparison} summarizes the distinctions between the proposed method and a subset of the methods reviewed in Section~\ref{sec:related_work} that are most directly related and relevant. Both bilevel BO and COBALT achieve the same input dimensionality reduction to $n_{\BB}$. The distinction is in the surrogate output (scalar vs.\ multi-output GP) and constraint handling (exact global optimization vs.\ moment approximation), not the input dimension. COBALT and BOCF surrogate intermediate black-box outputs $y$ and propagate uncertainty through the white-box equations, requiring multi-output GPs and moment approximations. Our method surrogates the scalar optimal objective $J(x^{\WB\star})$ directly, at the cost of an inner global optimization solve per outer iteration. When the white-box model is moderately sized and the black-box evaluation is the primary computational expense, this overhead is negligible. The approaches are complementary: COBALT handles $y$-dependent constraints natively and enables risk-averse decisions via uncertainty propagation; BOCF handles non-separable composite functions. Our method trades their generality for simplicity and exactness when Assumption~\ref{ass:separability} holds.

%% ============================================================================
%% SECTION 5: BENCHMARK SUITE
%% ============================================================================
\section{Benchmark Problem Suite}\label{sec:benchmark_suite}

We propose a suite of 13 problems spanning synthetic test functions and engineering domains, summarized in Table~\ref{tab:benchmark_summary}. Full mathematical formulations are in Appendix~\ref{sec:test_problems}. The suite is designed around five principles: (1)~every problem has a verified global optimum (Appendix~\ref{sec:optimum_verification}); (2)~constraints are physically motivated; (3)~black-box functions encode realistic structure-property relationships (Sabatier volcano curves, Langmuir isotherms, Robeson upper bounds, Hansen solubility); (4)~problems span a range of difficulty (unconstrained to 3 constraints, 2D to 5D, min and max); and (5)~all problems are cheap to evaluate, enabling the 8{,}450-run statistical comparison in Section~\ref{sec:results}.

To our knowledge, no existing benchmark suite targets a separable grey-box structure that reflects Assumption~\ref{ass:separability}: an outer black-box mapping $y = f^{\BB}(x^{\BB})$ coupled with an inner white-box optimization over $x^{\WB}$, subject to inequality constraints. Existing problem suites for composite BO~\cite{astudillo2019bocf} do not separate variables; those used in grey-box BO~\cite{paulson2022cobalt} contain only a handful of problems; and standard global optimization testbeds such as the IEEE Congress on Evolutionary Computation (CEC) benchmarks~\cite{suganthan2005cec} and the Comparing Continuous Optimizers (COCO) platform with its Black-Box Optimization Benchmarking (BBOB) suite~\cite{hansen2021coco} lack bilevel structure entirely.

\begin{table}[htbp]
\centering
\caption{Summary of the 13-problem benchmark suite. $n_{\WB}$: white-box variables, $n_{\BB}$: black-box variables, $n_y$: black-box outputs, $n_g$: inequality constraints.}
\label{tab:benchmark_summary}
\resizebox{\textwidth}{!}{%
\begin{tabular}{@{}lccccll@{}}
\toprule
\textbf{Problem} & $n_{\WB}$ & $n_{\BB}$ & $n_y$ & $n_g$ & \textbf{Type} & \textbf{Domain} \\
\midrule
Small-Feasible-Region 1 & 1 & 1 & 1 & 1 & min & Synthetic \\
Small-Feasible-Region 2 & 1 & 1 & 1 & 1 & min & Synthetic \\
Rastrigin & 2 & 1 & 1 & 0 & min & Multimodal \\
Toy-Hydrology & 1 & 1 & 1 & 2 & min & Hydrology \\
Rosen-Suzuki & 2 & 2 & 2 & 3 & min & Constrained \\
\midrule
CSTR & 3 & 2 & 3 & 2 & max & Catalysis \\
Heat-Exchanger & 3 & 2 & 3 & 2 & min & Heat Exchanger Network \\
PSA & 3 & 2 & 3 & 2 & min & Adsorption \\
Batch-Reactor & 3 & 2 & 3 & 2 & min & Kinetics \\
Distillation & 3 & 2 & 3 & 3 & min & Separation \\
Evaporator & 3 & 2 & 3 & 2 & min & Evaporation \\
Membrane & 3 & 2 & 3 & 2 & min & Membrane sep.\ \\
Williams-Otto & 3 & 2 & 2 & 2 & max & Process opt.\ \\
\bottomrule
\end{tabular}}
\end{table}

\subsection{Representative Problems}\label{subsec:representative_problems}

We highlight two problems; full formulations for all 13 appear in Appendix~\ref{sec:test_problems}.
% \noindent\textbf{Small-Feasible-Region Peovlwma (Appendix~\ref{subsec:app_sfr1}).}
Adapted from Gardner et al.~\cite{gardner2014} and Ariafar et al.~\cite{Ariafar2019}, these 2D constrained minimization problems feature small, disconnected feasible regions. The two-dimensional structure permits direct visualization of the search behavior (Section~\ref{subsec:sfr_results}).

% \noindent\textbf{CSTR with Catalyst Design (Appendix~\ref{subsec:app_cstr}).}
% This 5D problem maximizes yield of intermediate $B$ in consecutive reactions $A \to B \to C$. Three white-box variables (temperature, pressure, residence time) are governed by Arrhenius kinetics and CSTR mass balances; two black-box variables (binding energy, activation-energy shift) map to kinetic parameters via a Sabatier volcano curve. Two constraints enforce byproduct limits and a safety envelope. This is representative of the eight engineering problems in the suite.

%% ============================================================================
%% SECTION 6: COMPUTATIONAL EXPERIMENTS
%% ============================================================================
\section{Computational Experiments}\label{sec:experiments}\label{sec:results}

We evaluate bilevel BO against full-space baselines on the 13-problem benchmark suite. Section~\ref{sec:experimental_setup} specifies the solvers, metrics, and hyperparameter sweep. Section~\ref{subsec:sfr_results} then uses the two Small-Feasible-Region problems to visualize how bilevel and black-box BO explore the search space differently. Section~\ref{subsec:main_result} reports final regret across all 13 benchmarks, convergence speed, and robustness to hyperparameter and dimension.

\subsection{Experimental Design}\label{sec:experimental_setup}

This section specifies the four solvers compared (Section~\ref{subsec:solvers_compared}) and the four metrics reported (Section~\ref{subsec:metrics}). Per-problem hyperparameter settings and full sweep results are deferred to Appendix~\ref{sec:hyperparameter_details}.

\subsubsection{Solvers Compared}\label{subsec:solvers_compared}

We compare four solvers that represent different ways of handling the grey-box structure in problem~\eqref{eq:full_problem}.

The first solver, \textbf{black-box BH}, applies SciPy's Basin-Hopping (BH) solver over the full variable space $[x^{\WB}, x^{\BB}]$ with an SLSQP local minimizer. Basin-Hopping performs a global search by alternating random perturbations (uniform steps clipped to variable bounds) with local SLSQP minimizations, accepting or rejecting each local minimum via a Metropolis criterion ($T=100$). This solver provides a global-search baseline: Basin-Hopping is gradient-free at the global level and explores the full feasible region. However, it evaluates $f^{\BB}$ at every local minimization, so the total number of black-box evaluations grows with the number of Basin-Hopping iterations.

The second solver, \textbf{black-box BO (BB-BO)}, applies standard Bayesian optimization over the full $(n_{\WB}+n_{\BB})$-dimensional space. A single GP surrogate models the mapping $[x^{\WB}, x^{\BB}] \mapsto J$. Constraints are handled via a penalty function. This solver ignores the separable structure entirely and serves as the BO baseline.

The third solver, \textbf{bilevel BO with multi-start SLSQP (Bi-BO (SLSQP))}, is the proposed method with a gradient-based inner solver. It places the GP surrogate over only the $n_{\BB}$-dimensional black-box space and solves the constrained white-box subproblem using 50 Latin hypercube restarts of SLSQP, returning the best feasible solution found, as described in Algorithm~\ref{alg:bilevel_bo}.

The fourth solver, \textbf{bilevel BO with Basin-Hopping (Bi-BO (BH))}, is identical to the third except that it uses Basin-Hopping (100 iterations, SLSQP local minimizer, Metropolis temperature $T=100$) as the inner-loop global solver. Both bilevel variants share the same outer-loop GP surrogate and EI acquisition function; they differ only in the inner-loop solver used for the white-box subproblem.

All four solvers use Expected Improvement (EI) as the acquisition function for the BO variants. All three BO solvers share the same ARD RBF kernel and GP configuration (Appendix~\ref{sec:gp_details}). We sweep over initial sample sizes $n_\text{init} \in \{5, 25, 50\}$ and exploration parameters $\xi \in \{0.001, \ldots, 1.0\}$, yielding 8{,}450 independent runs across all 13 problems (8{,}190 BO + 130 NLP + 130 BH; Appendix~\ref{sec:hyperparameter_details}).

\subsubsection{Metrics}\label{subsec:metrics}

We report four metrics. \emph{Simple regret} $r_t = |J_\text{best}(t) - J^\star|$ measures the gap between the best objective found after $t$ evaluations and the verified global optimum $J^\star$. \emph{Convergence speed} is the number of iterations needed to reduce regret to 1\% of its initial value. \emph{Wall time} is the total elapsed time for $n_\text{init} + 200$ evaluations. Finally, we count the total number of \emph{black-box evaluations} (calls to $f^{\BB}$) to assess sample efficiency.

\subsection{Mechanism Visualization}\label{subsec:sfr_results}

We begin with the two Small-Feasible-Region problems because their two-dimensional structure permits direct visualization of how bilevel BO searches differently from black-box BO. Both problems have one white-box variable ($x^{\WB}$) and one black-box variable ($x^{\BB}$), with small, disconnected feasible regions defined by nonlinear constraints.

% SFR-1 search figure
\begin{figure}[htbp]
    \centering
    \includegraphics[width=\textwidth]{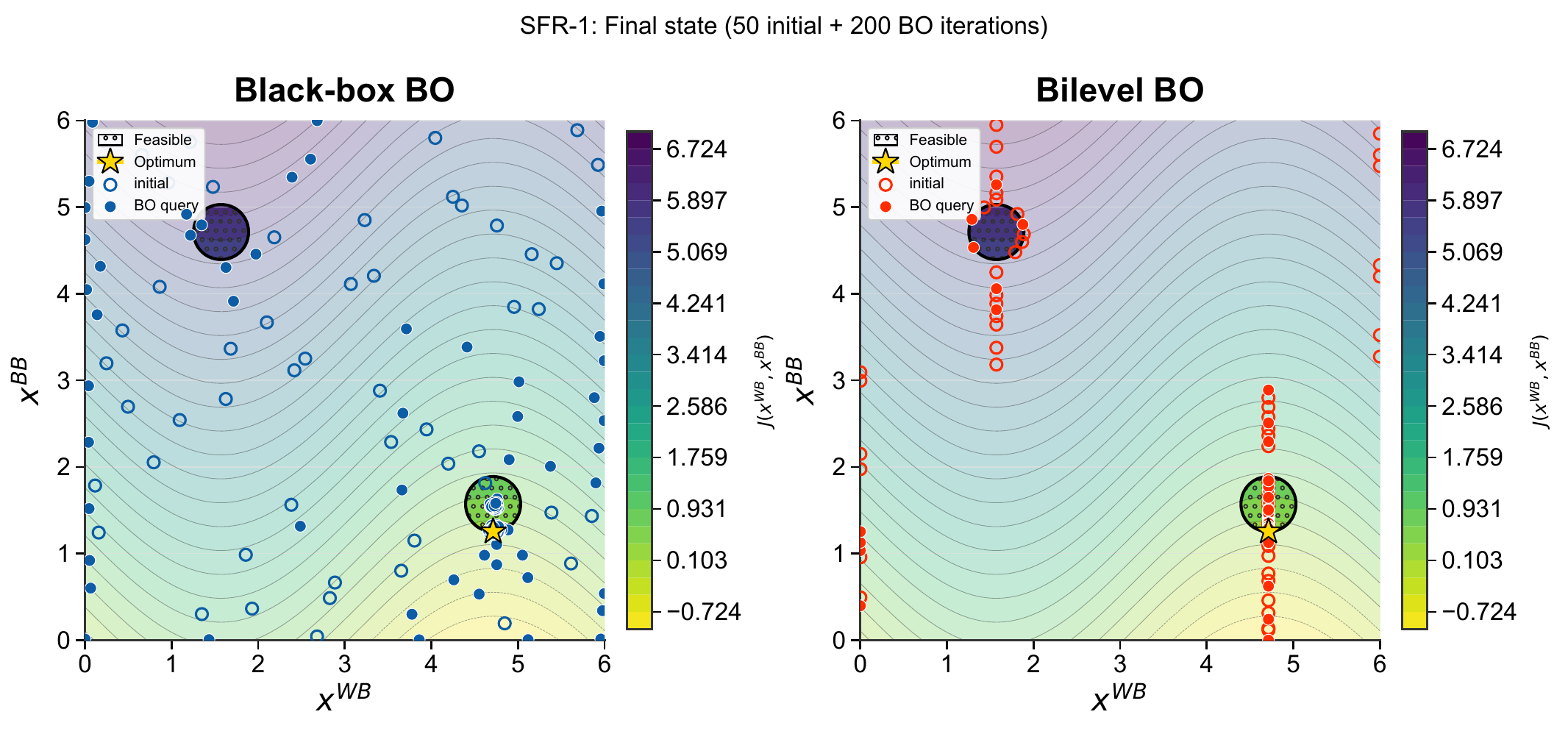}
    \caption{Small-Feasible-Region~1: search behavior of black-box BO (left) and bilevel BO (right) after 50 initial samples + 200 BO iterations. The colored background shows the objective value $J(x^{\WB}, y)$ where feasible (yellow = low, purple = high). White regions are \emph{infeasible}: no combination of $(x^{\WB}, x^{\BB})$ in those areas satisfies the constraints. The hatched islands are the only feasible regions, and the $\star$ marks the global optimum within them. Black-box BO scatters queries across the full $(x^{\WB}, x^{\BB})$ space, wasting many evaluations in infeasible white regions. Bilevel BO operates along the $x^{\BB}$ axis only---each query selects an $x^{\BB}$ value, and the inner BH solver finds the best feasible $x^{\WB}$, concentrating evaluations in or near the feasible islands.}
    \label{fig:sfr1_search}
\end{figure}

Figure~\ref{fig:sfr1_search} shows the key difference. In SFR-1, the black-box function is the identity ($y = x^{\BB}$), so the bilevel decomposition reduces to: given a candidate $y$-value from the outer loop, find the $x^{\WB}$ that minimizes $\sin(x^{\WB}) + y$ within the feasible set. The inner BH solver handles constraint satisfaction exactly, so every bilevel BO query lands in a feasible region. Black-box BO, by contrast, must jointly learn the objective function landscape \emph{and} the constraint boundary in two dimensions, wasting many samples in infeasible regions.

% SFR-2 search figure
\begin{figure}[htbp]
    \centering
    \includegraphics[width=\textwidth]{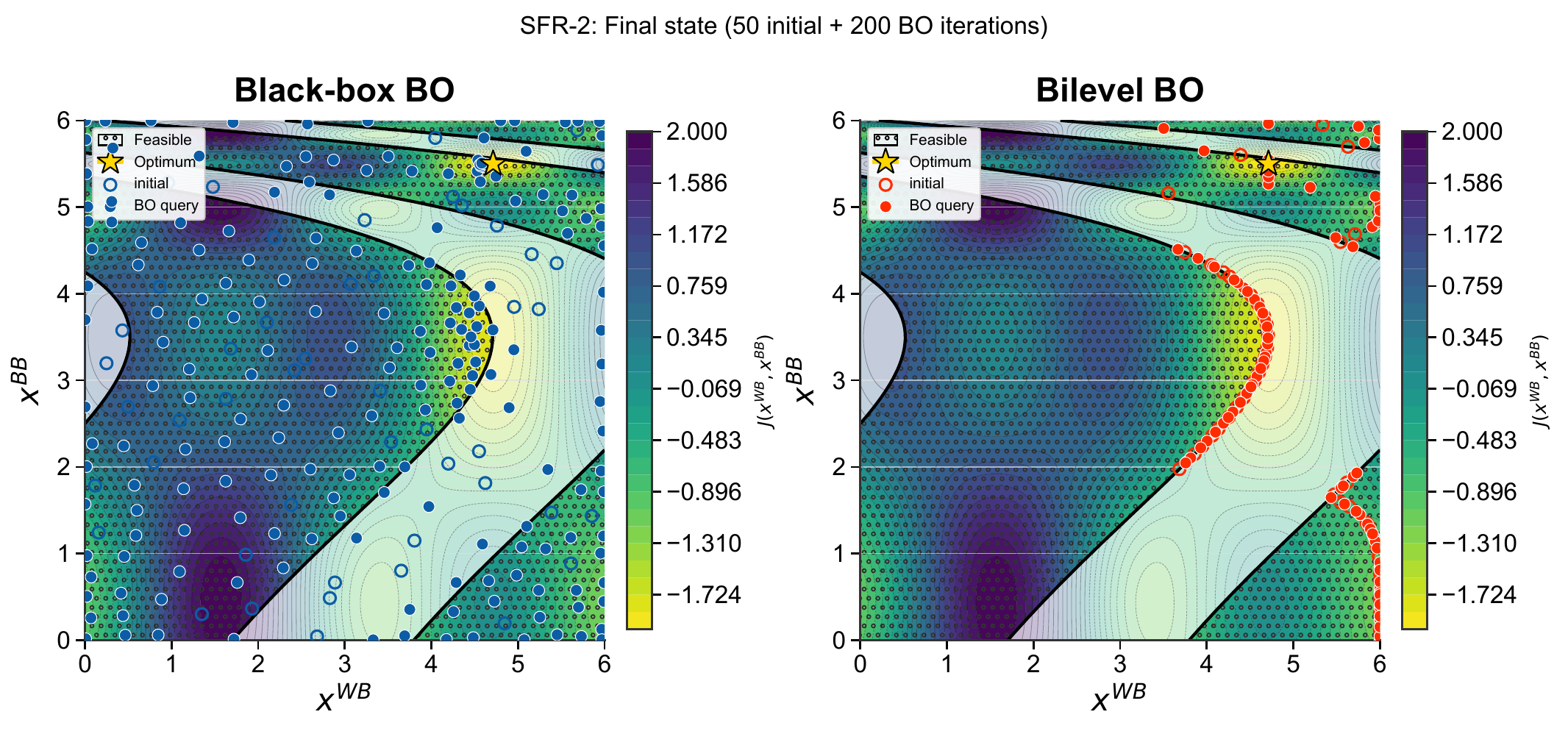}
    \caption{Small-Feasible-Region~2: same layout as Figure~\ref{fig:sfr1_search}. Here the black-box function is nonlinear ($y = \frac{x^{\BB}-5.5}{2}\exp(x^{\BB}/2)$), distorting the relationship between $x^{\BB}$ and the feasible set. Bilevel BO still operates near the optimum because the GP surrogate learns the scalar mapping $x^{\BB} \mapsto J(x^{\WB\star})$ in one dimension.}
    \label{fig:sfr2_search}
\end{figure}

SFR-2 (Figure~\ref{fig:sfr2_search}) introduces a nonlinear black-box function, making the mapping from $x^{\BB}$ to the feasible objective landscape more complex. The same mechanism applies: the outer loop proposes $x^{\BB}$, the black box function is used to compute $y$, and the inner BH solver finds the best feasible $x^{\WB}$. The GP surrogate now models a one-dimensional function $x^{\BB} \mapsto J(x^{\WB\star})$ rather than the two-dimensional objective-plus-constraint landscape that black-box BO must learn.

% SFR convergence curves
\begin{figure}[htbp]
    \centering
    \includegraphics[width=\textwidth]{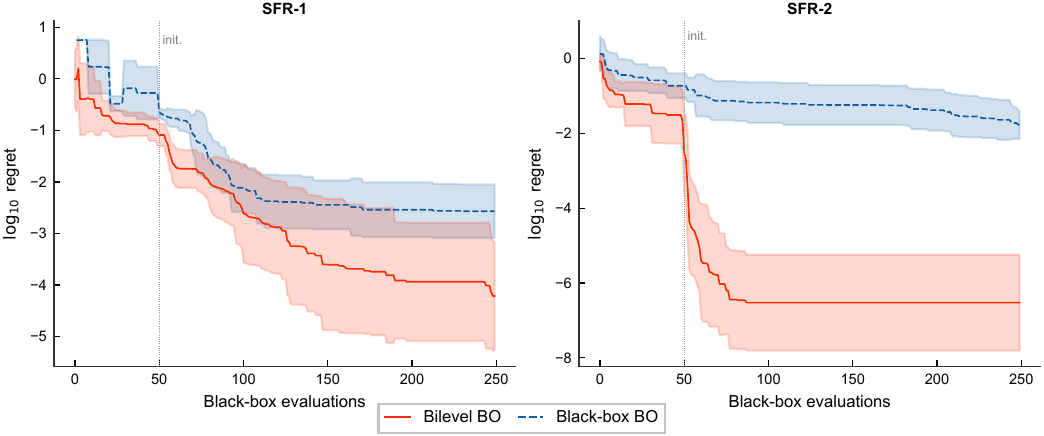}
    \caption{Regret convergence for the two Small-Feasible-Region problems. Each curve uses $n_\text{init}=50$ initial Latin hypercube samples and the EI exploration parameter $\xi$ that yielded the lowest final regret. The dashed line marks the end of initialization. Solid lines show the mean over 10 independent repetitions; shaded bands show $\pm 1$ standard deviation. Bilevel BO (red) converges 1--2 orders of magnitude below black-box BO (blue) on both problems.}
    \label{fig:sfr_convergence}
\end{figure}

The convergence curves (Figure~\ref{fig:sfr_convergence}) confirm the visual impression. On SFR-1, Bi-BO (SLSQP) reduces regret by an order of magnitude relative to black-box BO---an 11$\times$ improvement. On SFR-2, Bi-BO (SLSQP) reaches the global optimum (regret $< 10^{-5}$) while black-box BO stalls near $10^{-1.6}$---a 7{,}872$\times$ improvement. In both cases, the mechanism is the same: bilevel BO needs only to learn a one-dimensional function, while black-box BO must simultaneously model a two-dimensional objective and discover the small feasible islands.

\subsection{Main Results}\label{subsec:main_result}

The pattern observed on the SFR problems extends to all 13 benchmarks. Table~\ref{tab:final_regret} reports final regret after 200 BO iterations; Figure~\ref{fig:convergence_grid} shows convergence curves for the remaining 11 problems.

% 11-panel convergence grid (excludes SFR-1 and SFR-2)
\begin{figure}[htbp]
    \centering
    \includegraphics[width=\textwidth]{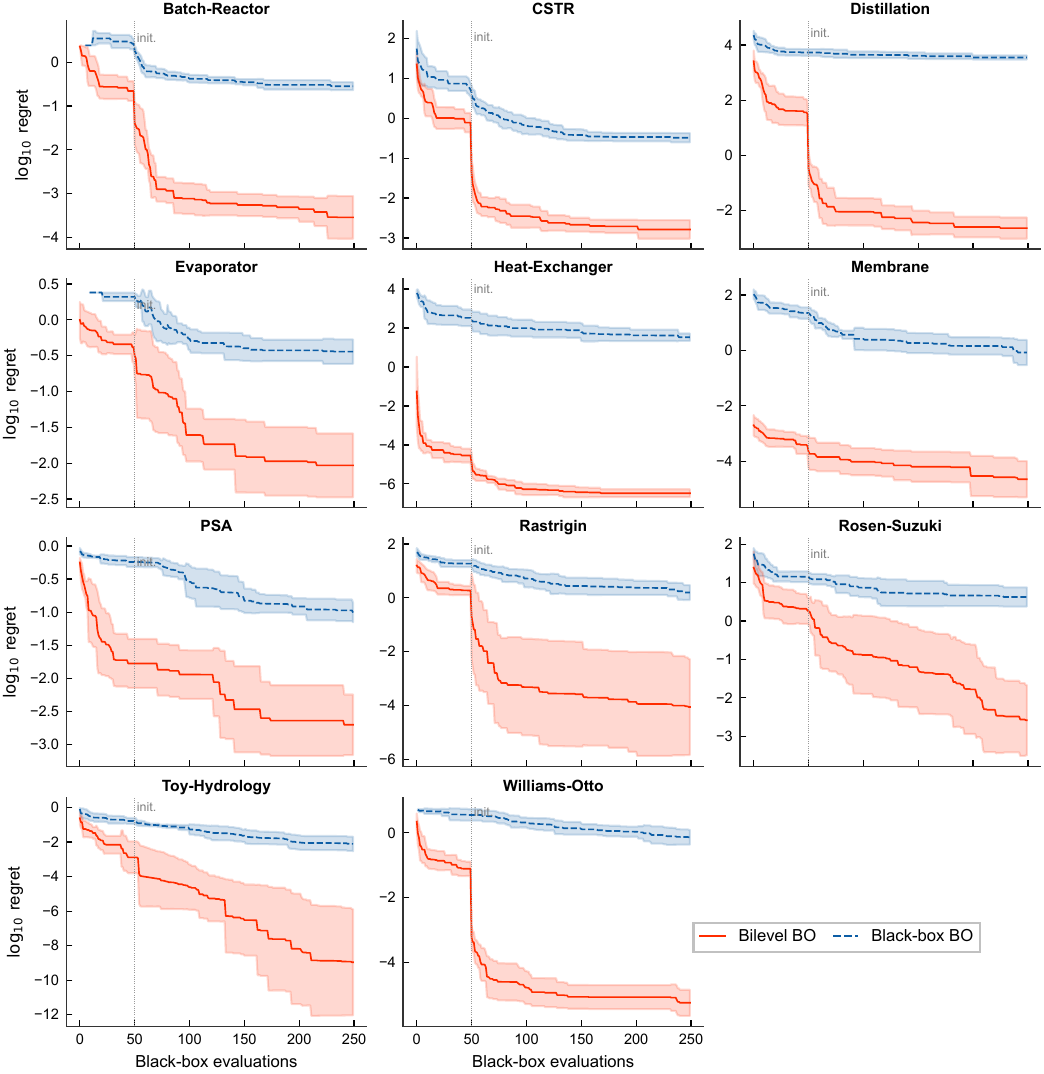}
    \caption{Regret convergence for the 11 benchmark problems not shown in Figure~\ref{fig:sfr_convergence}. Bi-BO (SLSQP) (red) converges to values that are 1--8 orders of magnitude below black-box BO (blue) on every problem. Each curve uses $n_\text{init}=50$ initial Latin hypercube samples and the EI exploration parameter $\xi$ that yielded the lowest final regret for that problem (see Appendix~\ref{sec:hyperparameter_details} for per-problem values). Solid lines show the mean over 10 independent repetitions; shaded bands show $\pm 1$ standard deviation.}
    \label{fig:convergence_grid}
\end{figure}

% Final regret summary table
\begin{table}[htbp]
\centering
\small
\caption{Mean final regret and mean wall time (seconds) after 200 iterations ($n_\text{init}=50$, best $\xi$ per method, 10 repetitions; see Appendix~\ref{sec:hyperparameter_details} for per-problem $\xi$ values). Bi-BO (SLSQP) achieves lower regret on all 13 problems; Bi-BO (BH) on 12 of 13 ($^\dagger$BH inner solver is worse than BB-BO on SFR-1).}
\label{tab:final_regret}
\resizebox{\textwidth}{!}{%
\begin{tabular}{@{}lc rr rr rr@{}}
\toprule
 & & \multicolumn{2}{c}{\textbf{BB-BO}} & \multicolumn{2}{c}{\textbf{Bi-BO (SLSQP)}} & \multicolumn{2}{c}{\textbf{Bi-BO (BH)}} \\
\cmidrule(lr){3-4} \cmidrule(lr){5-6} \cmidrule(lr){7-8}
\textbf{Problem} & \textbf{dim} & \textbf{Regret} & \textbf{Time (s)} & \textbf{Regret} & \textbf{Time (s)} & \textbf{Regret} & \textbf{Time (s)} \\
\midrule
Small-Feasible-Region   & 2 & 0.0048  & 207  & 0.0004 & 141 & 0.0334$^\dagger$ & 826   \\
Small-Feasible-Region-2 & 2 & 0.0243  & 197  & 0.0000 & 122 & 0.0000           & 150   \\
Rastrigin               & 3 & 1.8464  & 246  & 0.0026 & 104 & 0.0000           & 168   \\
Toy-Hydrology           & 2 & 0.0114  & 216  & 0.0000 & 116 & 0.0000           & 173   \\
Rosen-Suzuki            & 4 & 4.8785  & 281  & 0.0243 & 233 & 0.0567           & 356   \\
\midrule
CSTR                    & 5 & 0.3381  & 326  & 0.0019 & 201 & 0.0016           & 265   \\
Heat-Exchanger          & 5 & 37.1276 & 284  & 0.0000 & 200 & 0.0000           & 230   \\
PSA                     & 5 & 0.1044  & 315  & 0.0030 & 184 & 0.0010           & 275   \\
Batch-Reactor           & 5 & 0.2894  & 301  & 0.0005 & 216 & 0.0003           & 339   \\
Distillation            & 5 & 3621.95 & 260  & 0.0034 & 182 & 0.0078           & 231   \\
Evaporator              & 5 & 0.3841  & 307  & 0.0145 & 212 & 0.0254           & 249   \\
Membrane                & 5 & 1.2410  & 274  & 0.0000 & 214 & 0.0000           & 274   \\
Williams-Otto           & 5 & 0.8251  & 273  & 0.0000 & 326 & 0.0000           & 3{,}130 \\
\bottomrule
\end{tabular}}
\end{table}

Both bilevel variants achieve lower regret than black-box BO. Bi-BO (SLSQP) wins on all 13 problems (11$\times$--$10^{8}\times$ lower; geometric mean 3{,}192$\times$) at wall time comparable to BB-BO. Bi-BO (BH) wins on 12 of 13 (geometric mean 4{,}840$\times$), but Basin-Hopping's random perturbations overshoot the narrow feasible region of SFR-1 and leave it worse than BB-BO there. Gains scale with constraint tightness: Heat-Exchanger ($10^{8}\times$) and Distillation ($10^{6}\times$) expose the penalty method's failure in five dimensions, while the smallest SLSQP gains (SFR-1: 11$\times$, Evaporator: 26$\times$) occur where the white-box subproblem itself has a narrow feasible region or multiple local optima.

Wall-time differences concentrate on problems with large infeasible regions in $x^{\BB}$. When a sampled $x^{\BB}$ admits no feasible process variables, Basin-Hopping keeps perturbing in search of a feasible point rather than returning quickly, inflating runtime to 826\,s on SFR-1 and 3{,}130\,s on Williams-Otto (4--10$\times$ the SLSQP inner solver). SLSQP exits on infeasibility and lets the outer BO loop advance.

The full-space BH baseline fails on Rastrigin (multimodal joint space) and performs poorly on Distillation; Bi-BO matches or approaches its quality on the remaining problems using $\sim$250 black-box evaluations versus the substantially higher counts BH requires. All 13 SLSQP pairwise comparisons are statistically significant (Wilcoxon signed-rank, Holm--Bonferroni adjusted $p < 0.05$; sign test 13/13, $p = 0.000244$); 12 of 13 hold for BH (sign test, $p = 0.0034$). A Friedman test across the four methods confirms differences ($\chi^2 = 17.03$, $p = 0.0007$), with Nemenyi post-hoc placing BB-BO below both Bi-BO variants and NLP (critical difference $= 1.30$).

When using BH as the inner solver for Bi-BO, the performance on the small feasible region problem (SFR-1) is worse than the black-box BO baseline. Similarly, the solution times for the small feasible region problem and the Williams-Otto problem are significantly higher for the Bi-BO (BH) variant compared to both the Bi-BO (SLSQP) variant and the black-box BO baseline. Both of these problems are characterized by narrow feasible regions, and regions which are completely infeasible for certain $x^{\BB}$ values. The random perturbations in the BH algorithm cause it to overshoot these narrow feasible regions, or to iterate until satisfied that no feasible solution exists. This highlights a key tradeoff between the two inner solvers: BH provides stronger global optimization guarantees but can be inefficient in problems with narrow or disconnected feasible regions, while SLSQP is more efficient but may struggle with multimodality. The choice of inner solver should therefore be informed by problem characteristics.

Figure~\ref{fig:scatter_bb_vs_bi} visualizes the relationship between final regret for black-box BO and Bi-BO (SLSQP) across all 273 (problem, $n_\text{init}$, $\xi$) configurations. Every point lies below the diagonal, confirming that Bi-BO outperforms BB-BO regardless of hyperparameter choice. The wide range of regret ratios (11$\times$--$10^{8}\times$) is visible in the vertical spread of points, with the largest gains occurring on problems with tight constraints and high dimensionality. The greater vertical spread in Bi-BO final regret (compared to the relatively narrow horizontal spread in BB-BO) reflects the sensitivity of the inner solver to problem landscape: on problems with multimodal or non-convex inner subproblems (e.g., Evaporator, Rosen-Suzuki), the multi-start SLSQP inner solver occasionally converges to local optima, producing variable final regret across hyperparameter configurations. On problems with well-behaved inner landscapes (e.g., Heat-Exchanger, Membrane), Bi-BO regret is consistently near zero.

% Scatter plot
\begin{figure}[htbp]
    \centering
    \includegraphics[width=0.5\textwidth]{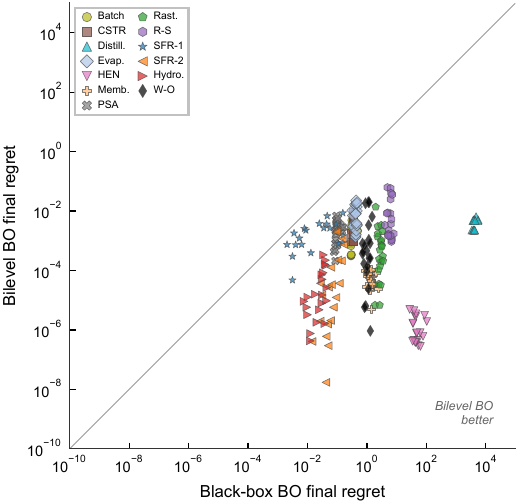}
    \caption{Scatter plot of final regret: black-box BO vs.\ bilevel BO (SLSQP variant). Each point is one (problem, $n_\text{init}$, $\xi$) configuration across 273 total configurations. All points lie below the diagonal, confirming that Bi-BO (SLSQP) outperforms BB-BO regardless of hyperparameter choice.}
    \label{fig:scatter_bb_vs_bi}
\end{figure}

%% ============================================================================
%% SECTION 7: DISCUSSION
%% ============================================================================
\section{Discussion}\label{sec:discussion}

Bi-BO achieves 11$\times$--$10^{8}\times$ lower regret than black-box BO on every problem tested (SLSQP variant; geometric mean 3{,}192$\times$), with equal or faster wall time on 12 of 13 problems. The source of this advantage is structural, not algorithmic: we use the same GP kernel, the same acquisition function, and the same optimizer as the black-box baseline. The gains come entirely from exploiting the separable problem structure---reducing the surrogate dimension and satisfying the white-box constraints exactly. Three additional analyses---feasibility rates, the relationship between regret ratio and problem characteristics, and sample efficiency crossover---provide mechanistic insight into when and why the bilevel decomposition helps.

\noindent\textbf{Constraint satisfaction and wasted evaluations.}
A key mechanism underlying the regret gap is that black-box BO wastes evaluations on infeasible queries. Figure~\ref{fig:feasibility_rates} reports the fraction of all 250 evaluations that land in feasible regions, reconstructed from stored search histories. Black-box BO wastes 6--72\% of its budget on infeasible points depending on the problem, with Small-Feasible-Region (72\%), Batch-Reactor (71\%), and Evaporator (58\%) being the worst cases. Bilevel BO achieves 100\% feasibility on 8 of 12 constrained problems because the inner BH solver handles constraint satisfaction exactly. On the remaining four problems (Batch-Reactor, Evaporator, Small-Feasible-Region, PSA), bilevel BO feasibility ranges from 59--91\%. This occurs for two reasons: (1)~the inner BH may return solutions where the total constraint violation $\max_i g_i(x^{\WB}, y)$ is positive but small (on the order of $10^{-4}$--$10^{-2}$), placing the solution near but outside the feasible boundary; or (2)~certain $y = f^{\BB}(x^{\BB})$ values yield an inner problem with no feasible solution, because no $x^{\WB} \in \mathcal{X}^{\WB}$ satisfies $g(x^{\WB}, y) \leq 0$. This is not a limitation of the bilevel architecture itself but of the inner solver; deterministic global solvers (e.g., BARON) would restore the 100\% guarantee.

% Feasibility rate figure
\begin{figure}[htbp]
    \centering
    \includegraphics[width=\textwidth]{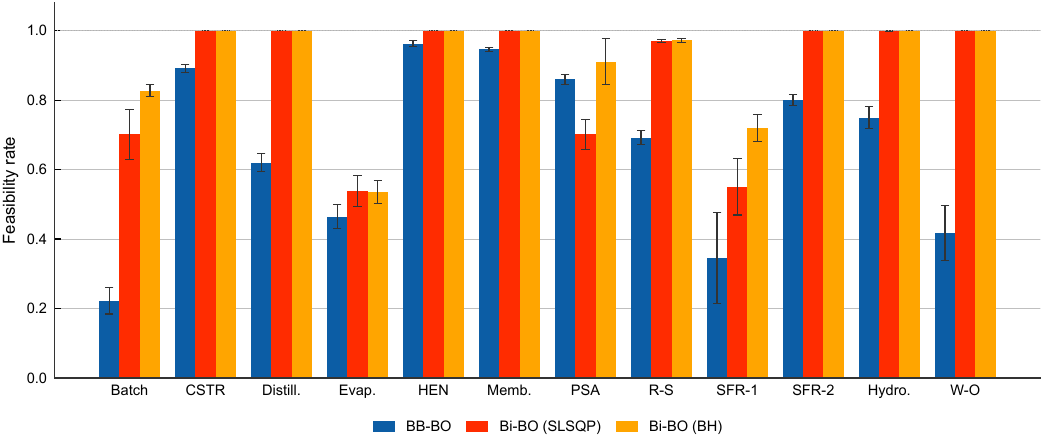}
    \caption{Fraction of evaluations landing in feasible regions across 12 constrained problems ($n_\text{init}=50$, best $\xi$, mean over 10 repetitions). Black-box BO (blue) wastes 6--72\% of its budget on infeasible queries; bilevel BO (red and orange) achieves $\geq$59\% feasibility on all problems and 100\% on 8 of 12.}
    \label{fig:feasibility_rates}
\end{figure}

\noindent\textbf{Factors influencing the regret ratio.}
Factors beyond raw problem dimensionality such as constraint geometry, inner-problem multimodality, and black-box landscape complexity determine the magnitude of improvement from BB-BO to Bi-BO. For example, Rastrigin (unconstrained, 3D total) achieves 698$\times$ improvement with the SLSQP inner solver (Table~\ref{tab:final_regret}) purely from reducing the GP surrogate from 3D to 1D. Meanwhile, Evaporator and Rosen-Suzuki have the smallest SLSQP ratios (26$\times$ and 201$\times$, respectively), because their non-convex inner landscapes cause the multi-start SLSQP to find local optima. Heat-Exchanger has a loose constraint set (93\% feasible) but the largest gain ($10^{8}\times$), driven by dimensionality reduction from 5D to 2D combined with a well-behaved inner landscape. The analysis confirms that when the inner optimizer converges reliably, dimensionality reduction alone produces large gains; exact constraint satisfaction provides additional benefit but is neither necessary nor sufficient.

\noindent\textbf{Sample efficiency.}
Beyond final regret, practitioners are interested in how quickly a method reaches useful solutions. Figure~\ref{fig:sample_crossover} reports the number of evaluations at which bilevel BO first matches the final regret that black-box BO achieves after all 250 evaluations. On 5 of 13 problems (Heat-Exchanger, Membrane, Williams-Otto, Distillation, CSTR), bilevel BO matches or surpasses black-box BO's final performance within the initial sampling phase alone---effectively from the first BO iteration. Even in the hardest case (Evaporator), crossover occurs at 47 evaluations (19\% of the total budget). This has direct practical implications: when each black-box evaluation costs hours of compute time, bilevel BO can deliver black-box BO-quality solutions using 80--100\% fewer expensive evaluations.

% Sample efficiency crossover
\begin{figure}[htbp]
    \centering
    \includegraphics[width=\textwidth]{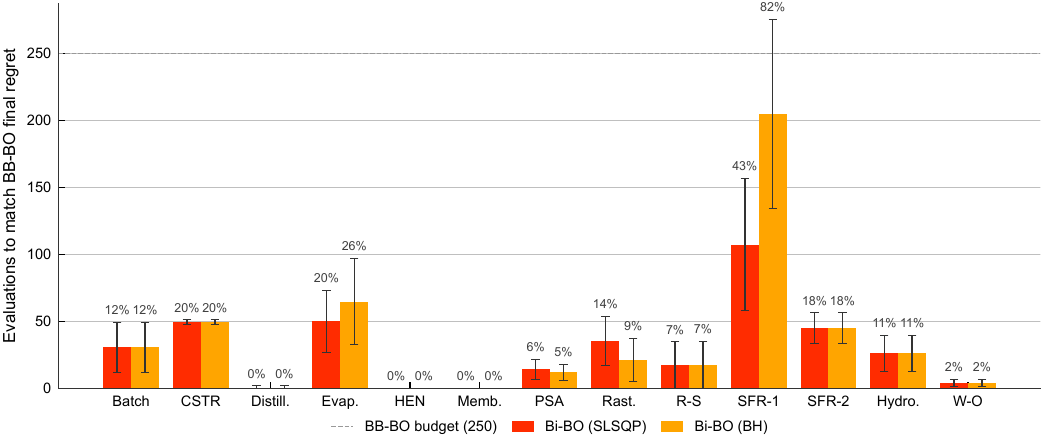}
    \caption{Sample efficiency crossover: number of evaluations for bilevel BO to match black-box BO's final regret (at 250 evaluations). Percentages show the fraction of the total budget required. On 10 of 13 problems, bilevel BO matches black-box BO's final performance within 10\% of the evaluation budget.}
    \label{fig:sample_crossover}
\end{figure}

\noindent\textbf{Connection to decomposition in global optimization.}
As discussed in Section~\ref{sec:related_work}, the bilevel reformulation belongs to a broader family of decomposition strategies---block coordinate descent~\cite{wright2015coordinate}, trust-region filter methods~\cite{eason2016trf,eason2018trf}, and data-driven bilevel solvers like DOMINO~\cite{beykal2020domino}. Our results provide empirical support for a principle from bilevel optimization theory~\cite{colson2007overview}: solving the inner problem exactly improves outer-level convergence. In the surrogate-based setting, the inner global optimizer eliminates noise from outer-level observations, giving the GP a cleaner signal. This suggests that bilevel decomposition may be beneficial whenever part of the optimization landscape admits an efficient exact solver, even beyond the grey-box problems studied here.

\noindent\textbf{Inner-solver choice: SLSQP vs.\ Basin-Hopping.}
The choice of inner-loop solver affects both solution quality and computational cost. Across all 13 problems, multi-start SLSQP and Basin-Hopping achieve comparable final regret (family-level sign test: SLSQP wins 6, BH wins 7, $p = 1.0$; geometric mean SLSQP/BH ratio: 1.52$\times$). However, the two solvers exhibit complementary strengths. BH excels on multimodal inner landscapes: on Rastrigin, Bi-BO (BH) achieves $7.2 \times 10^{7}\times$ improvement over BB-BO, while Bi-BO (SLSQP) achieves 698$\times$---a gap of five orders of magnitude attributable entirely to SLSQP's inability to escape local optima in the multimodal Rastrigin inner problem. Conversely, SLSQP excels on problems with narrow feasible regions: on SFR-1, Bi-BO (SLSQP) achieves 11$\times$ improvement while Bi-BO (BH) is 0.1$\times$ (\emph{worse} than BB-BO), because Basin-Hopping's random perturbations overshoot the small feasible set. SLSQP is also faster on all 13 problems (geometric mean 1.74$\times$), making it the recommended default. BH is preferred when the inner landscape is known or suspected to be multimodal.

\noindent\textbf{Complementarity with existing grey-box BO.}
As detailed in Section~\ref{sec:related_work}, COBALT~\cite{paulson2022cobalt} and BOCF~\cite{astudillo2019bocf} address settings our method does not: non-separable composite functions and risk-averse decisions via uncertainty propagation. Our method trades their generality for simplicity and exactness when Assumption~\ref{ass:separability} holds. A direct comparison on the same benchmark suite is planned for future work.

\noindent\textbf{Baseline constraint handling.}
The black-box BO baseline uses a fixed penalty ($10^6$) for constraint violations. More sophisticated constrained BO methods---feasibility-weighted expected improvement~\cite{gardner2014}, augmented Lagrangian approaches~\cite{gramacy2016}---would likely narrow the gap on tightly constrained problems such as Heat-Exchanger and Distillation, where the penalty method is least effective. However, the dimensionality reduction advantage of bilevel BO would persist regardless of constraint handling, as the surrogate dimension reduction from $n_{\WB}+n_{\BB}$ to $n_{\BB}$ is independent of the acquisition function. We intentionally use vanilla methods for both BO and NLP to isolate the effect of the bilevel structure rather than specific algorithmic choices.

\noindent\textbf{Limitations and open problems.}
We identify six limitations, each pointing to an open problem. (1)~The inner optimizer must converge to the global optimum (Assumption~\ref{ass:separability}(ii)). For non-convex inner problems, the inner solver may converge to local optima---as observed in Section~\ref{sec:results}, this explains the modest 11$\times$ improvement on SFR-1 (SLSQP variant) and the sub-100\% feasibility on Batch-Reactor and Evaporator (Figure~\ref{fig:feasibility_rates}). Replacing the stochastic inner solver with a deterministic global solver (e.g., BARON) would provide optimality certificates but at higher computational cost. (2)~The inner solver can perform \emph{worse} than the black-box baseline on problems with very narrow feasible regions: on SFR-1, the BH inner solver's random perturbations generate predominantly infeasible starting points, yielding a 0.1$\times$ ratio. Multi-start SLSQP avoids this failure mode on SFR-1 but lacks BH's global exploration for multimodal problems. (3)~Expensive inner optimizations can dominate wall time. With the BH inner solver, Williams-Otto requires 3{,}130s vs.\ 273s for BB-BO (11.5$\times$ overhead); even the faster SLSQP variant incurs a 1.19$\times$ overhead on this problem. Warm-starting the inner solver from previous solutions or using reduced-order inner models could mitigate this. (4)~Pure black-box constraints---those depending directly on $f^{\BB}$ outputs without passing through the white-box model---are not handled; extending the outer loop with constrained BO methods~\cite{gardner2014,gelbart2014unknown} would address this. (5)~All black-box functions in our suite are closed-form surrogates of what would be expensive computations; validation on a real DFT or molecular dynamics black box would strengthen the practical case. (6)~Formal convergence guarantees for the outer BO loop---which follow from standard results since it is BO over a compact domain with a GP surrogate---are not analyzed here; characterizing how the bilevel structure affects dimension-dependent regret bounds is future work.

\noindent\textbf{Broader implications.}
Expensive black-box optimization problems frequently contain exploitable structure---known physics, analytical submodels, or differentiable components---that monolithic surrogates ignore. The bilevel decomposition is one instance of a general principle: \emph{solve what you can solve exactly, and approximate only what you must}. As surrogate-based methods scale to higher-dimensional search spaces, structural decomposition becomes not merely helpful but necessary to overcome the curse of dimensionality. The 13-problem benchmark suite provides a testbed for developing and comparing methods that exploit this structure.

%% ============================================================================
%% SECTION 8: CONCLUSION
%% ============================================================================
\section{Conclusion}\label{sec:conclusion}

When grey-box optimization problems exhibit variable separability---the expensive black-box depends only on one variable group while the remaining variables enter known, differentiable equations---solving the known subproblem exactly and approximating only the black-box component yields 11$\times$--$10^{8}\times$ lower regret than monolithic surrogate-based optimization. The advantage is robust to hyperparameters and inner-solver choice, and requires no algorithmic novelty beyond the decomposition itself. A suite of 13 benchmark problems with verified global optima provides a standardized testbed for future methods that exploit this structure.

The opening observation of this paper---that monolithic surrogates waste samples learning what is already known---is nearly self-evident once the separable structure is recognized. Yet no prior method systematically exploits this structure for dimensionality reduction with exact constraint satisfaction (in principle; subject to inner-solver global optimality in practice). The empirical evidence across 8{,}450 independent optimization runs confirms that the payoff is large and consistent. As expensive black-box optimization scales to higher dimensions, identifying and exploiting known substructure will be essential; the bilevel decomposition studied here is one principled way to do so.

%% ============================================================================
%% BACK MATTER
%% ============================================================================
\section*{Acknowledgements}

Computational resources were provided by the Texas Advanced Computing Center (TACC) at The University of Texas at Austin.

\section*{Declarations}

\begin{itemize}
\item \textbf{Funding.} This work was supported by the ExxonMobil, whose financial support is acknowledged with gratitude.

\item \textbf{Conflict of interest.} The authors declare that they have no conflict of interest.

\item \textbf{Data availability.} All data generated during this study are reproducible from the code described below. No external datasets were used.

\item \textbf{Code availability.} The Python code for the benchmark suite, solvers, and experiment runner is available at \url{https://github.com/joshuaeh/HybridBayesianOptimization}.

\item \textbf{Author contributions.} Joshua E.\ Hammond: conceptualization, methodology, software, investigation, writing---original draft. Tyler A.\ Soderstrom: methodology, investigation, writing---review \& editing. Brian A.\ Korgel: supervision, writing---review \& editing. Michael Baldea: conceptualization, supervision, funding acquisition, writing---review \& editing.

\item \textbf{Use of generative AI.} Generative AI tools (Claude Opus 4.6, Anthropic) were used to assist with code development, data analysis, and manuscript drafting. All AI-generated content was reviewed, verified, and edited by the authors, who take full responsibility for the accuracy and integrity of the work. No AI tool was used to generate scientific claims, interpret results, or design experiments without human oversight.
\end{itemize}

%% ============================================================================
%% APPENDICES
%% ============================================================================
\begin{appendices}
\section{Test Problem Definitions}\label{sec:test_problems}

Full mathematical formulations for all 13 benchmark problems. Each problem specifies the white-box and black-box models, constraints, variable bounds, and verified global optimum.

\subsection{Small Feasible Region Problem 1}\label{subsec:app_sfr1}
Gardner et al.~\cite{gardner2014} and Ariafar et al.~\cite{Ariafar2019} propose a constrained optimization test problem with a small feasible region that can be formulated as follows:
\begin{align}
    \min\limits_{x} & \sin{x_1} + y_1, \\\notag
    \text{s.t.} &~~ g_1(x) := \sin(x_1)\sin(y_1) + 0.95 \leq 0, \\\notag
    &~~ y_1 = f^\BB(x) := x_2, \\\notag
    &~~ 0 \leq x_i \leq 6, ~~~~ \forall i \in \{1, 2\}
\end{align}

The global minimum is $J^\star = 0.2532$ and is located at $x^{\WB\star}_1=\frac{3\pi}{2}$, $x^{\BB\star}=x_2=1.2532$.

\subsection{Modified Small Feasible Region Problem 2}\label{subsec:app_sfr2}

Gardner et al.~\cite{gardner2014} and Ariafar et al.~\cite{Ariafar2019} propose a constrained optimization test problem with a small feasible region. We add an additional black-box function $y_1=f^\BB (x_2)$ that forms a grey-box problem as follows:
\begin{align} \label{eq:SmallFeasibleRegion}
    \min\limits_{x} & \cos(2x_1)\cos(y_1) + \sin(x_1), \\\notag
    \text{s.t.} &~~ g_1(x) := \cos(x_1)\cos(y_1) - \sin(x_1)\sin(y_1) - 0.5 \leq 0, \\\notag
    &~~ y_1 = f^\BB(x) := \frac{(x_2 - 5.5)}{2} \exp\left(\frac{x_2}{2}\right), \\\notag
    &~~ 0 \leq x_i \leq 6, ~~~~ \forall i \in \{1, 2\}
\end{align}

The global minimum $J^\star = -2.0$ is located at $x^{\BB\star}=5.5, x^{\WB\star}=3\pi/2$.

\subsection{Rastrigin}\label{subsec:app_rastrigin}

We use a modified formulation of the Rastrigin function~\cite{rastrigin1974} that can be formulated as a multi-scale grey-box problem as follows:
\begin{align} \label{eq:Rastrigin}
    \min\limits_{x, y} &~~ 30 + x_1^2 - 10\cos(2\pi x_1) + x_2^2 - 10\cos(2\pi x_2) + y_1, \\\notag
    \text{s.t.} &~~ y_1 = f^\BB(x) := x_3^2 - 10\cos( 2 \pi x_3), \\\notag
    &~~ -5.12 \leq x_i \leq 5.12, ~~~~~ \forall i \in \{ 1, 2,3 \}, \\\notag
    &~~ x = [x^{\WB}, x^{\BB}], ~~~~~ x^{\WB} = [x_1, x_2], ~~~~~ x^{\BB} = [x_3].
\end{align}
The global minimum is equal to $0$ with $x^{\WB\star} = [0, 0], x^{\BB\star} = [0]^\top$.

\subsection{Toy-Hydrology}\label{subsec:app_hydrology}

We use a modified formulation of the Toy Hydrology problem~\cite{gramacy2016} that can be formulated as a multi-scale grey-box problem as follows:
\begin{align} \label{eq:ToyHydrology}
    \min\limits_{x, y} &~~ x_1 + x_2, \\\notag
    \text{s.t.} &~~ g_1(x, y) := 1.5 - x_1 - 2x_2 - 0.5\sin(-4\pi x_2 + y_1) \leq 0, \\\notag
    &~~ g_2(x, y) := x_1^2 + x_2^2 - 1.5 \leq 0, \\\notag
    &~~ y_1 = f^\BB(x_1) := 2 \pi x_1^2, \\\notag
    &~~ 0 \leq x_i \leq 1, ~~~~~ \forall i \in \{ 1, 2 \}, \\\notag
    &~~ x = [x^{\WB}, x^{\BB}], ~~~~~ x^{\BB} = [x_1], ~~~~~ x^{\WB} = [x_2].
\end{align}
The global minimum is $0.5998$ with $x^{\WB\star} = 0.4047, x^{\BB\star} = 0.1951$.

\subsection{Rosen-Suzuki}\label{subsec:app_rosen_suzuki}

We use a modified formulation of the Rosen-Suzuki function~\cite{rosenbrock1960} that can be formulated as a multi-scale grey-box problem as follows:
\begin{align} \label{eq:RosenSuzuki}
    \min\limits_{x, y} &~~ x_1^2 + x_2^2 + x_4^2 - 5x_1 - 5x_2 + y_1, \\\notag
    \text{s.t.} &~~ g_1(x, y) := -(8 - x_1^2 - x_2^2 - x_3^2 - x_4^2 - x_1 + x_2 - x_3 + x_4) \leq 0, \\\notag
    &~~ g_2(x, y) := -(10 - x_1^2 - 2x_2^2 - y_2 + x_1 + x_4) \leq 0, \\\notag
    &~~ g_3(x, y) := -(5 - 2x_1^2 - x_2^2 - x_3^2 - 2x_1 + x_2 + x_4) \leq 0, \\\notag
    &~~ y_1 = f^\BB_1(z) := 2x_3^2 - 21x_3 + 7x_4, \\\notag
    &~~ y_2 = f^\BB_2(z) := x_3^2 + 2x_4^2, \\\notag
    &~~ -2 \leq x_i \leq 2, ~~~~~ \forall i \in \{ 1,\ldots, 4 \} \\\notag
    &~~ x = [x^{\WB}, x^{\BB}], ~~~~~ x^{\WB} = [x_1, x_2], ~~~~~ x^{\BB} = [x_3, x_4].
\end{align}
The global minimum is $-44$ with $x^{\WB\star} = [0, 1]^\top, x^{\BB\star} = [2, -1]^\top$.

\subsection{Catalytic Reactor Design}\label{subsec:app_cstr}
We use a catalytic reactor design problem that combines a white-box CSTR process model with black-box catalyst properties akin to those derived from DFT calculations. The objective is to \emph{maximize} the yield of intermediate product $B$ in consecutive reactions $A \to B \to C$. The problem can be formulated as a multi-scale grey-box problem as follows:
\begin{align} \label{eq:CSTR}
    \max\limits_{T, P, \tau, E_b, \Delta E_a} &~~ Y_B \\\notag
    \text{s.t.} &~~ \bff^\WB=\left[
        \begin{array}{l}
            Y_B-\frac{k_1 \tau}{\left(1+k_1 \tau\right)\left(1+k_2 \tau\right)} \\
            k_1\left(T, P, E_b, \Delta E_a\right)-A_1 v\left(E_b\right) \exp \left(-\frac{E_{a, 1}}{R T}\right)\left(\frac{P}{P_{\mathrm{ref}}}\right)^{\beta_1} \frac{1}{\left(1+K_{\mathrm{ads}} P\right)^2} \\
            k_2\left(T, P, E_b, \Delta E_a\right)-A_2 \exp \left(-\frac{E_{a, 2}}{R T}\right)\left(\frac{P}{P_{\mathrm{ref}}}\right)^{\beta_2} \frac{1}{\left(1+K_{\mathrm{ads}} P\right)^2} \\
            K_{\mathrm{ads}}\left(E_b\right) - K_0 \exp\left(\kappa \left(E_b - E_b^*\right)\right)
        \end{array}\right], \\\notag
    &~~ \bff^\BB=\left[
        \begin{array}{l}
            v\left(E_b\right)-\left[\exp \left(-\frac{1}{2}\left(\frac{E_b-E_b^*}{\sigma_E}\right)^2\right)\right] \\
            E_{a, 1}-\left[E_{a, 1}^{(0)}-\alpha_1 \Delta E_a\right] \\
            E_{a, 2}-\left[E_{a, 2}^{(0)}+\alpha_2 \Delta E_a\right]
        \end{array}\right], \\\notag
    &~~ g_1(x, y) := Y_C - 0.25 \leq 0, \\\notag
    &~~ g_2(x, y) := \frac{P}{P_{\max}} - \frac{T - T_{\min}}{T_{\max} - T_{\min}} - 0.3 \leq 0,
\end{align}

and the conversion and byproduct yield are:
\begin{align}\notag
    X_A &= \frac{k_1 \tau}{1 + k_1 \tau}, & Y_C &= X_A - Y_B.
\end{align}

The black-box function $f^{\BB}$ maps catalyst properties $\Delta E_a$, $E_b$ to kinetic parameters via a volcano-type activity model:

\begin{table}[htbp]
    \caption{Parameters for the catalytic reactor design problem.\label{tab:cstr_parameters}}
\begin{tabular}{llll}
    \toprule
    \bf{Symbol} & \bf{Description} & \bf{Range / Value} & \bf{Units} \\
    \midrule
    \multicolumn{4}{l}{Decision / design variables} \\
    \midrule
    $T$ & Temperature & [450, 700] & K \\
    $P$ & Pressure & [1, 20] & bar \\
    $\tau$ & Residence time & [0.1, 5.0] & s \\
    $E_b$ & Binding energy & [-2.0, 0.0] & eV \\
    $\Delta E_a$ & Activation-energy shift & [-0.2, 0.2] & - \\
    \midrule
    $A_1$ & Pre-exponential (step 1) & $1.0 \times 10^7$ & $\mathrm{s}^{-1}$ \\
    $A_2$ & Pre-exponential (step 2) & $3.0 \times 10^6$ & $\mathrm{s}^{-1}$ \\
    $E_{a 1}\left(\Delta E_a\right)$ & Activation energy & $80,000-20,000 \Delta E_a$ & J/mol \\
    $E_{a 2}\left(\Delta E_a\right)$ & Activation energy & $95,000+5,000 \Delta E_a$ & J/mol \\
    $P_{\text {ref }}$ & Reference pressure & 8.0 & bar \\
    $E_b^{\star}$ & Reference binding energy & -1.0 & eV \\
    $\sigma_E$ & Width parameter & 0.40 & eV \\
    $K_0$ & Adsorption rate pre-exponential factor & 0.06 & 1/bar \\
    $\kappa$ & Adsorption rate exponential term & 3.0 & $1 / \mathrm{eV}$ \\
    $R$ & Gas constant & 8.314 & J/(mol K) \\
    $k_B$ & Boltzmann constant & $8.617 \times 10^{-5}$ & eV/K \\
    $\alpha_1$ & & 20,000 & - \\
    $\alpha_2$ & & 5,000 & - \\
    $\beta_1$ & Exponent for pressure dependence & 1.0 & - \\
    $\beta_2$ & Exponent for pressure dependence & 0.35 & - \\
    \bottomrule
\end{tabular}
\end{table}
The global maximum yield given the parameters in Table~\ref{tab:cstr_parameters} is $92.32\%$ with $x^{\WB\star} = [647.3, 20.0, 5.0]^\top$, $x^{\BB\star} = [-1.047, 0.200]^\top$.

\subsection{Heat Exchanger with Novel Working Fluid}\label{subsec:app_heat_exchanger}

We consider the design of a heat exchanger system where the working fluid's thermodynamic properties (heat capacity $c_p$, viscosity $\mu$, and thermal conductivity $k$) depend on two molecular descriptors ($m_1$, $m_2$) in a black-box function. White-box equations govern the heat exchanger design and operation (heat transfer area $A$, mass flow rate $\dot{m}$, and log-mean temperature difference $\Delta T_{lm}$) with the goal of minimizing the total cost (capital + operating) while satisfying minimum heat duty~\cite{yee1990simultaneous}. Capital cost scales with area according to the six-tenths rule~\cite{peters1991plant} while operating cost is a function of pumping power which depends on fluid viscosity and mass flow rate according to the relation provided by~\cite{mizutani2003mathematical}. We use a simplified form for the overall heat transfer coefficient $U$ based the resistance model from Ravagnani and Caballero~\cite{ravagnani2007minlp}.

The problem can be formulated as:
\begin{align} \label{eq:HeatExchanger}
    \min\limits_{A, \dot{m}, \Delta T_{lm}, m_1, m_2} &~~ C_{\text{capital}} + C_{\text{operating}} \\\notag
    \text{s.t.} &~~ C_{\text{capital}} = 1000 \cdot {x^{\WB}_1}^{0.6}, \\\notag
    &~~ C_{\text{operating}} = 500 \cdot \frac{{x^{\WB}_2}^3 \cdot y_2}{x^{\WB}_1}, \\\notag
    &~~ U = \frac{y_3}{0.01 + 0.001/y_1}, \\\notag
    &~~ g_1(x, y) := 100 - U \cdot x^{\WB}_1 \cdot x^{\WB}_3 \leq 0, \\\notag
    &~~ g_2(x, y) := 4000 - \frac{4x^{\WB}_2}{\pi \cdot 0.05 \cdot y_2} \leq 0, \\\notag
    &~~ y_1 = f^\BB_1(x^{\BB}_1, x^{\BB}_2) := 2.0 + 0.5\sin\left(\frac{\pi x^{\BB}_1}{50}\right) + 0.3x^{\BB}_2, \\\notag
    &~~ y_2 = f^\BB_2(x^{\BB}_1, x^{\BB}_2) := 0.001 \cdot \exp\left(\frac{x^{\BB}_1 - 100}{50}\right) \cdot (1 + 0.5x^{\BB}_2), \\\notag
    &~~ y_3 = f^\BB_3(x^{\BB}_1, x^{\BB}_2) := 0.1 + 0.05\left(1 - \left(\frac{x^{\BB}_1 - 125}{75}\right)^2\right) + 0.02x^{\BB}_2, \\\notag
    &~~ A \in [1, 50], \quad \dot{m} \in [0.1, 5], \quad \Delta T_{lm} \in [5, 50], \\\notag
    &~~ m_1 \in [50, 200], \quad m_2 \in [0, 1], \\\notag
    &~~ x = [x^{\WB}, x^{\BB}], \quad x^{\WB} = [A, \dot{m}, \Delta T_{lm}], \quad x^{\BB} = [m_1, m_2],
\end{align}

\begin{table}[htbp]
\centering
\caption{Variables, physical meaning, bounds, and units (Heat exchanger)}
\label{tab:heat_exchanger_vars}
\begin{tabular}{@{}llccl@{}}
\toprule
Variable & Physical meaning & Lower & Upper & Units \\
\midrule
$x^{\WB}_1$ & Heat transfer area ($A$) & 1 & 50 & m$^2$ \\
$x^{\WB}_2$ & Mass flow rate ($\dot{m}$) & 0.1 & 5 & kg/s \\
$x^{\WB}_3$ & Log-mean temperature difference ($\Delta T_{lm}$) & 5 & 50 & K \\
\midrule
$x^{\BB}_1$ & Molecular descriptor ($m_1$) & 50 & 200 & (unitless) \\
$x^{\BB}_2$ & Molecular descriptor ($m_2$) & 0 & 1 & (unitless) \\
\midrule
$y_1$ & Heat capacity ($c_p$) & N/A & N/A & kJ/kg$\cdot$K \\
$y_2$ & Viscosity ($\mu$) & N/A & N/A & Pa$\cdot$s \\
$y_3$ & Thermal conductivity ($k$) & N/A & N/A & W/m$\cdot$K \\
\bottomrule
\end{tabular}
\end{table}

where $y_1$ is the heat capacity [kJ/kg$\cdot$K], $y_2$ is the viscosity [Pa$\cdot$s], and $y_3$ is the thermal conductivity [W/m$\cdot$K]. The constraint $g_1$ ensures the heat duty $Q = U \cdot A \cdot \Delta T_{lm} \geq 100$ kW is satisfied, and $g_2$ enforces turbulent flow with Reynolds number $\text{Re} \geq 4000$. The global minimum cost given the parameters in Table~\ref{tab:heat_exchanger_vars} is $1000.0$ with $x^{\WB\star} = [1.0, 0.1, 19.3]^\top$, $x^{\BB\star} = [50.0, 0.0]^\top$.

\subsection{Pressure Swing Adsorption with Novel Adsorbent}\label{subsec:app_psa}

We consider the design of a pressure swing adsorption (PSA) cycle for gas separation, a problem that naturally exhibits grey-box structure~\cite{boukouvala2017argonaut,agarwal2010superstructure,ruthven1994psa}. The adsorbent properties---maximum loading $q_{\max}$, adsorption equilibrium constant for the target component $K_A$, and equilibrium constant for the competing species $K_B$---are computed from molecular simulations~\cite{burns2020process,smit2008molecular} that serve as black-box functions of adsorbent structural parameters (pore size $\sigma$ and interaction energy $\epsilon$)~\cite{dubbeldam2019force}.

The white-box model captures the PSA cycle performance through Langmuir isotherm-based working capacity calculations~\cite{langmuir1918adsorption,wang2023adsorption}. The working capacity $\Delta q$ represents the difference in loading between high-pressure adsorption and low-pressure desorption steps~\cite{yang1987gas,ruthven1994psa}. Selectivity $\alpha = K_A/K_B$ determines separation performance~\cite{hasan2013cost}, while productivity is defined as working capacity divided by adsorption time~\cite{leperi2019surrogate}. The objective minimizes the trade-off between maximizing productivity and minimizing compression costs, where isothermal compression work scales with $\ln(P_H/P_L)$~\cite{smith2018thermodynamics}.

The problem can be formulated as:
\begin{align} \label{eq:PSA}
    \min\limits_{P_H, P_L, t_{\text{ads}}, \sigma, \epsilon} &~~ -\text{Prod} + 0.1 \cdot \ln\left(\frac{P_H}{P_L}\right) \\\notag
    \text{s.t.} &~~ \alpha = \frac{y_2}{y_3}, \\\notag
    &~~ \Delta q = y_1 \cdot \left(\frac{y_2 \cdot P_H}{1 + y_2 \cdot P_H} - \frac{y_2 \cdot P_L}{1 + y_2 \cdot P_L}\right), \\\notag
    &~~ \text{Prod} = \frac{\Delta q}{t_{\text{ads}}}, \\\notag
    &~~ g_1(x, y) := 3 - \alpha \leq 0, \\\notag
    &~~ g_2(x, y) := 0.95 - \frac{\alpha \cdot P_H}{1 + \alpha \cdot P_H} \leq 0, \\\notag
    &~~ y_1 = f^\BB_1(\sigma, \epsilon) := 5.0 \cdot \exp\left(-\frac{(\sigma - 6)^2}{4}\right) \cdot \left(1 + 0.1\epsilon\right), \\\notag
    &~~ y_2 = f^\BB_2(\sigma, \epsilon) := 0.5 \cdot \exp\left(\frac{\epsilon - 10}{5}\right) \cdot \left(1 - 0.05(\sigma - 5)^2\right), \\\notag
    &~~ y_3 = f^\BB_3(\sigma, \epsilon) := 0.1 \cdot \exp\left(\frac{\epsilon - 15}{8}\right) \cdot \left(1 + 0.03(\sigma - 7)^2\right), \\\notag
    &~~ P_H \in [2, 10], \quad P_L \in [0.1, 1], \quad t_{\text{ads}} \in [10, 120], \\\notag
    &~~ \sigma \in [3, 10], \quad \epsilon \in [5, 25], \\\notag
    &~~ x = [x^{\WB}, x^{\BB}], \quad x^{\WB} = [P_H, P_L, t_{\text{ads}}], \quad x^{\BB} = [\sigma, \epsilon],
\end{align}
The black-box functions $f^\BB$ model the adsorbent properties as functions of the Lennard-Jones parameters: pore size $\sigma$ [\AA] and interaction energy $\epsilon$ [kJ/mol]~\cite{dubbeldam2019force,smit2008molecular}. These relationships capture how molecular-scale adsorbent design affects macroscopic separation performance~\cite{burns2020process}.

\begin{table}[htbp]
\centering
\caption{Variables, physical meaning, bounds, and units (Pressure Swing Adsorption)}
\label{tab:psa_vars}
\begin{tabular}{@{}llll@{}}
\toprule
\textbf{Symbol} & \textbf{Description} & \textbf{Range / Value} & \textbf{Units} \\
\midrule
\multicolumn{4}{l}{White-box decision variables} \\
\midrule
$P_H$ & High (adsorption) pressure & [2, 10] & bar \\
$P_L$ & Low (desorption) pressure & [0.1, 1] & bar \\
$t_{\text{ads}}$ & Adsorption step time & [10, 120] & s \\
\midrule
\multicolumn{4}{l}{Black-box decision variables} \\
\midrule
$\sigma$ & Adsorbent pore size & [3, 10] & \AA \\
$\epsilon$ & Lennard-Jones interaction energy & [5, 25] & kJ/mol \\
\midrule
\multicolumn{4}{l}{Black-box outputs} \\
\midrule
$y_1 = q_{\max}$ & Maximum loading capacity & --- & mol/kg \\
$y_2 = K_A$ & Adsorption equilibrium constant (component A) & --- & 1/bar \\
$y_3 = K_B$ & Adsorption equilibrium constant (component B) & --- & 1/bar \\
\bottomrule
\end{tabular}
\end{table}

The constraint $g_1$ ensures minimum selectivity $\alpha \geq 3$ for effective separation~\cite{hasan2013cost,yang1987gas}, and $g_2$ enforces minimum product purity based on the Langmuir isotherm equilibrium~\cite{langmuir1918adsorption,ruthven1994psa}. The global minimum given the parameters in Table~\ref{tab:psa_vars} is $-0.6433$ with $x^{\WB\star} = [3.97, 0.1, 10.0]^\top$, $x^{\BB\star} = [6.04, 19.55]^\top$.

\subsection{Batch Reactor with Catalyst Optimization}\label{subsec:app_batch_reactor}

We consider the optimization of a batch reactor for consecutive reactions $A \to B \to C$, where the objective is to maximize the yield of the intermediate product $B$~\cite{parulekar1988yield}. This classic selectivity problem~\cite{luus1999towards} gains a multi-scale dimension when catalyst properties must be co-optimized with operating conditions.

The white-box model consists of standard Arrhenius kinetics~\cite{fogler2016elements} governing the batch reactor mass balances, which admit analytical solutions for first-order consecutive reactions. Temperature $T$, batch time $t_f$, and initial concentration $C_{A,0}$ are the white-box decision variables.

The black-box model implements a Sabatier volcano relationship~\cite{motagamwala2021microkinetic,ooka2021sabatier,wodrich2024microkinetic} that captures how catalyst activity depends on binding energy $E_b$. The volcano principle states that optimal catalysts bind reactants neither too strongly (limiting desorption) nor too weakly (limiting adsorption), with peak activity at an intermediate binding energy. The catalyst particle size $d$ provides an additional degree of freedom that affects the pre-exponential factor through available surface area.

The problem can be formulated as:
\begin{align} \label{eq:BatchReactor}
    \min\limits_{T, t_f, C_{A,0}, E_b, d} &~~ -C_B + 0.001 \cdot T \cdot t_f \\\notag
    \text{s.t.} &~~ k_1' = y_1 \cdot \exp\left(\frac{-2000}{T}\right), \\\notag
    &~~ k_2' = y_2 \cdot \exp\left(\frac{-2400}{T}\right), \\\notag
    &~~ C_A = C_{A,0} \cdot \exp(-k_1' \cdot t_f), \\\notag
    &~~ C_B = C_{A,0} \cdot \frac{k_1'}{k_2' - k_1'} \cdot \left(\exp(-k_1' \cdot t_f) - \exp(-k_2' \cdot t_f)\right), \\\notag
    &~~ g_1(x, y) := 0.8 - \left(1 - \frac{C_A}{C_{A,0}}\right) \leq 0, \\\notag
    &~~ g_2(x, y) := 0.7 - \frac{C_B}{C_{A,0} - C_A} \leq 0, \\\notag
    &~~ y_1 = f^\BB_1(E_b, d) := 10 \cdot d \cdot \exp\left(-\frac{(E_b + 1)^2}{0.25}\right), \\\notag
    &~~ y_2 = f^\BB_2(E_b, d) := 0.5 \cdot d \cdot \exp\left(-\frac{(E_b + 0.5)^2}{0.5}\right), \\\notag
    &~~ y_3 = f^\BB_3(E_b, d) := \exp\left(-2 \cdot E_b - 1\right), \quad \text{(not used in white-box equations)} \\\notag
    &~~ T \in [300, 500], \quad t_f \in [0.5, 10], \quad C_{A,0} \in [0.5, 5], \\\notag
    &~~ E_b \in [-2, 0], \quad d \in [0.5, 2], \\\notag
    &~~ x = [x^{\WB}, x^{\BB}], \quad x^{\WB} = [T, t_f, C_{A,0}], \quad x^{\BB} = [E_b, d],
\end{align}
The Gaussian functions in $f^\BB_1$ and $f^\BB_2$ encode the volcano-shaped activity curves: $k_1$ peaks at $E_b = -1$ eV while $k_2$ peaks at $E_b = -0.5$ eV, creating a selectivity landscape where catalyst design can favor the desired reaction over the consecutive side reaction. The output $y_3$ does not appear in any white-box equation or constraint; it is included intentionally to test robustness to irrelevant black-box outputs, since in practice molecular simulations often return quantities not all of which enter the process model.

\begin{table}[htbp]
\centering
\caption{Variables, physical meaning, bounds, and units (Batch Reactor)}
\label{tab:batch_reactor_vars}
\begin{tabular}{@{}llll@{}}
\toprule
\textbf{Symbol} & \textbf{Description} & \textbf{Range / Value} & \textbf{Units} \\
\midrule
\multicolumn{4}{l}{White-box decision variables} \\
\midrule
$T$ & Reactor temperature & [300, 500] & K \\
$t_f$ & Batch (final) time & [0.5, 10] & h \\
$C_{A,0}$ & Initial concentration of A & [0.5, 5] & mol/L \\
\midrule
\multicolumn{4}{l}{Black-box decision variables} \\
\midrule
$E_b$ & Catalyst binding energy & [$-2$, 0] & eV \\
$d$ & Catalyst particle size factor & [0.5, 2] & --- \\
\midrule
\multicolumn{4}{l}{Black-box outputs} \\
\midrule
$y_1 = k_1^0$ & Pre-exponential factor for $A \to B$ & --- & 1/s \\
$y_2 = k_2^0$ & Pre-exponential factor for $B \to C$ & --- & 1/s \\
$y_3 = K_{\text{eq}}$ & Adsorption equilibrium constant & --- & --- \\
\bottomrule
\end{tabular}
\end{table}

The constraint $g_1$ ensures minimum conversion $X_A \geq 0.8$, and $g_2$ enforces selectivity $S_B = C_B/(C_{A,0} - C_A) \geq 0.7$, requiring that at least 70\% of converted reactant forms the desired intermediate. The global minimum given the parameters in Table~\ref{tab:batch_reactor_vars} is $-1.7488$ with $x^{\WB\star} = [500.0, 4.394, 5.0]^\top$, $x^{\BB\star} = [-1.01, 2.0]^\top$.

\subsection{Extractive Distillation with Novel Solvent}\label{subsec:app_distillation}

We consider the integrated design of an extractive distillation column and entrainer selection, a canonical example of simultaneous solvent and process optimization~\cite{kossack2008systematic,scheffczyk2018cosmo}. The white-box model employs classical shortcut methods: the Fenske equation~\cite{fenske1932fractionation} determines minimum stages $N_{\min}$ from the relative volatility and separation specifications, the Underwood equations~\cite{underwood1948fractional,kamath2010equation} yield minimum reflux $R_{\min}$, and operating constraints follow from the Gilliland correlation~\cite{gilliland1940multicomponent,grossmann2005optimal}.

The black-box model represents COSMO-RS or other molecular thermodynamic predictions that map solvent molecular descriptors---here parameterized by Hansen solubility parameters for hydrogen bonding $\delta_H$ and polar interactions $\delta_P$---to the modified relative volatility $\alpha_{12}^{\text{mod}}$ in the presence of the entrainer~\cite{bruggemann2004shortcut}. The entrainer also affects operating costs through its density $\rho$ (determining column sizing) and viscosity $\mu$ (affecting tray hydraulics and heat transfer).

The problem can be formulated as:
\begin{align} \label{eq:Distillation}
    \min\limits_{R, N, S/F, \delta_H, \delta_P} &~~ 1000 \cdot N + 500 \cdot R + 200 \cdot (S/F) \cdot y_3 \\\notag
    \text{s.t.} &~~ R_{\min} = \frac{1}{y_1 - 1} \cdot \left(\frac{x_D}{x_F} - y_1 \cdot \frac{1 - x_D}{1 - x_F}\right), \\\notag
    &~~ N_{\min} = \frac{\ln\left(\frac{x_D(1-x_B)}{x_B(1-x_D)}\right)}{\ln(y_1)}, \\\notag
    &~~ g_1(x, y) := 1.2 \cdot R_{\min} - R \leq 0, \\\notag
    &~~ g_2(x, y) := 1.5 \cdot N_{\min} - N \leq 0, \\\notag
    &~~ g_3(x, y) := (S/F) \cdot y_2 - 5000 \leq 0, \\\notag
    &~~ y_1 = f^\BB_1(\delta_H, \delta_P) := 2.5 + 1.0 \cdot \sin\left(\frac{\pi(\delta_H - 10)}{20}\right) \cdot \cos\left(\frac{\pi(\delta_P - 10)}{15}\right), \\\notag
    &~~ y_2 = f^\BB_2(\delta_H, \delta_P) := 800 + 50 \cdot \delta_H - 20 \cdot \delta_P, \\\notag
    &~~ y_3 = f^\BB_3(\delta_H, \delta_P) := 0.5 + 0.1 \cdot \delta_H + 0.05 \cdot \delta_P^2 / 100, \\\notag
    &~~ R \in [1, 10], \quad N \in [5, 50], \quad S/F \in [0.5, 5], \\\notag
    &~~ \delta_H \in [5, 25], \quad \delta_P \in [5, 20], \\\notag
    &~~ x = [x^{\WB}, x^{\BB}], \quad x^{\WB} = [R, N, S/F], \quad x^{\BB} = [\delta_H, \delta_P],
\end{align}
where $x_D = 0.99$ (distillate purity), $x_F = 0.5$ (feed composition), and $x_B = 0.01$ (bottoms impurity).

\begin{table}[htbp]
\centering
\caption{Variables, physical meaning, bounds, and units (Extractive Distillation)}
\label{tab:distillation_vars}
\begin{tabular}{@{}llll@{}}
\toprule
\textbf{Symbol} & \textbf{Description} & \textbf{Range / Value} & \textbf{Units} \\
\midrule
\multicolumn{4}{l}{White-box decision variables} \\
\midrule
$R$ & Reflux ratio & [1, 10] & --- \\
$N$ & Number of theoretical stages & [5, 50] & --- \\
$S/F$ & Solvent-to-feed ratio & [0.5, 5] & --- \\
\midrule
\multicolumn{4}{l}{Black-box decision variables} \\
\midrule
$\delta_H$ & Hansen hydrogen-bonding parameter & [5, 25] & MPa$^{0.5}$ \\
$\delta_P$ & Hansen polar parameter & [5, 20] & MPa$^{0.5}$ \\
\midrule
\multicolumn{4}{l}{Black-box outputs} \\
\midrule
$y_1 = \alpha_{12}^{\text{mod}}$ & Modified relative volatility & --- & --- \\
$y_2 = \rho$ & Solvent density & --- & kg/m$^3$ \\
$y_3 = \mu$ & Solvent viscosity & --- & cP \\
\midrule
\multicolumn{4}{l}{Fixed parameters} \\
\midrule
$x_D$ & Distillate purity specification & 0.99 & mol/mol \\
$x_F$ & Feed composition & 0.50 & mol/mol \\
$x_B$ & Bottoms impurity specification & 0.01 & mol/mol \\
\bottomrule
\end{tabular}
\end{table}

The constraint $g_1$ ensures the reflux ratio exceeds the minimum by at least 20\% for operational flexibility, $g_2$ requires sufficient stages above the Fenske minimum, and $g_3$ limits solvent inventory based on density to control capital costs. The global minimum cost for the parameters in Table~\ref{tab:distillation_vars} is $11{,}758$ with $x^{\WB\star} = [1.0, 11.0, 0.5]^\top$, $x^{\BB\star} = [19.8, 9.99]^\top$.

\subsection{Multi-Effect Evaporator with Phase-Change Material}\label{subsec:app_evaporator}

We consider a prototype test problem for the design of a multi-effect evaporator (MEE) system integrated with a novel phase-change material (PCM) as the heat transfer medium. Multi-effect evaporator synthesis follows established energy balance formulations~\cite{hillenbrand1988synthesis}, where steam economy improves with additional effects but with diminishing returns due to tighter temperature driving forces. The efficiency degradation across effects is captured by the factor $0.9^{n-1}$, consistent with typical industrial values of 0.7--0.9 pounds of evaporation per pound of steam per effect~\cite{swenson2024mee}.

The white-box model includes energy balances relating the heat transfer fluid (HTF) flow rate and properties to evaporation capacity~\cite{ahmetovic2018simultaneous}. The latent heat of vaporization for water (2260 kJ/kg) determines the evaporation rate from available thermal energy. The number of effects $n$, HTF mass flow rate $\dot{m}_{\text{HTF}}$, and temperature drop per effect $\Delta T_{\text{effect}}$ are the white-box decision variables.

The black-box model represents molecular simulation or group-contribution predictions of PCM thermophysical properties~\cite{sharma2009review,zalba2003review} as functions of melting temperature $T_m$ and a molecular structure parameter $\lambda$ that controls the length of alkyl chains (affecting both latent heat and heat capacity). The key outputs are latent heat of fusion $\Delta H_{\text{fus}}$ (typical range 150--300 kJ/kg for organic PCMs), liquid heat capacity $c_p^{\text{liq}}$ (typically 2--3 kJ/(kg$\cdot$K)), and the actual melting temperature $T_{\text{melt}}$. The functional forms in $f^\BB$ are intended to capture qualitative structure-property relationships rather than precise physical models.

The problem can be formulated as:
\begin{align} \label{eq:Evaporator}
    \min\limits_{n, \dot{m}_{\text{HTF}}, \Delta T_{\text{effect}}, T_m, \lambda} &~~ -\dot{m}_{\text{evap}} + 0.1 \cdot n^2 + 0.01 \cdot \dot{m}_{\text{HTF}}^2 \\\notag
    \text{s.t.} &~~ Q_{\text{HTF}} = \dot{m}_{\text{HTF}} \cdot (y_1 + y_2 \cdot 10), \\\notag
    &~~ \dot{m}_{\text{evap}} = \frac{Q_{\text{HTF}} \cdot n \cdot 0.9^{n-1}}{2260}, \\\notag
    &~~ g_1(x, y) := 100 + n \cdot \Delta T_{\text{effect}} - y_3 \leq 0, \\\notag
    &~~ g_2(x, y) := 5 - \dot{m}_{\text{evap}} \leq 0, \\\notag
    &~~ y_1 = f^\BB_1(T_m, \lambda) := 150 \cdot \lambda \cdot \left(1 + 0.2 \sin\left(\frac{\pi T_m}{100}\right)\right), \\\notag
    &~~ y_2 = f^\BB_2(T_m, \lambda) := 2.0 + 0.5 \cdot \lambda - 0.002 \cdot T_m, \\\notag
    &~~ y_3 = f^\BB_3(T_m, \lambda) := T_m + 10 \cdot \sin\left(\frac{\pi \lambda}{2}\right), \\\notag
    &~~ n \in [2, 6] \text{ (treated as continuous; integer rounding would apply in practice)}, \\\notag
    &~~ \dot{m}_{\text{HTF}} \in [1, 20], \quad \Delta T_{\text{effect}} \in [5, 20], \\\notag
    &~~ T_m \in [50, 150], \quad \lambda \in [0.5, 2], \\\notag
    &~~ x = [x^{\WB}, x^{\BB}], \quad x^{\WB} = [n, \dot{m}_{\text{HTF}}, \Delta T_{\text{effect}}], \quad x^{\BB} = [T_m, \lambda],
\end{align}
\begin{table}[htbp]
\centering
\caption{Variables, physical meaning, bounds, and units (Multi-Effect Evaporator)}
\label{tab:evaporator_vars}
\begin{tabular}{@{}llll@{}}
\toprule
\textbf{Symbol} & \textbf{Description} & \textbf{Range / Value} & \textbf{Units} \\
\midrule
\multicolumn{4}{l}{White-box decision variables} \\
\midrule
$n$ & Number of evaporator effects & [2, 6] & --- \\
$\dot{m}_{\text{HTF}}$ & Heat transfer fluid mass flow rate & [1, 20] & kg/s \\
$\Delta T_{\text{effect}}$ & Temperature drop per effect & [5, 20] & $^\circ$C \\
\midrule
\multicolumn{4}{l}{Black-box decision variables} \\
\midrule
$T_m$ & Target melting temperature & [50, 150] & $^\circ$C \\
$\lambda$ & Molecular structure parameter & [0.5, 2] & --- \\
\midrule
\multicolumn{4}{l}{Black-box outputs} \\
\midrule
$y_1 = \Delta H_{\text{fus}}$ & Latent heat of fusion & --- & kJ/kg \\
$y_2 = c_p^{\text{liq}}$ & Liquid phase heat capacity & --- & kJ/(kg$\cdot$K) \\
$y_3 = T_{\text{melt}}$ & Actual melting temperature & --- & $^\circ$C \\
\midrule
\multicolumn{4}{l}{Fixed parameters} \\
\midrule
$\Delta H_{\text{vap,water}}$ & Latent heat of water vaporization & 2260 & kJ/kg \\
$\eta_{\text{effect}}$ & Per-effect efficiency factor & 0.9 & --- \\
\bottomrule
\end{tabular}
\end{table}

The constraint $g_1$ ensures temperature feasibility: the PCM melting temperature must exceed the cumulative temperature drop across all effects plus a 100$^\circ$C baseline. The constraint $g_2$ enforces a minimum evaporation rate of 5 kg/s to meet production requirements. The global minimum for parameters in Table~\ref{tab:evaporator_vars} is $-1.9587$ with $x^{\WB\star} = [4.09, 19.07, 5.0]^\top$, $x^{\BB\star} = [120.47, 2.0]^\top$.

\subsection{Membrane Separation with Designed Polymer}\label{subsec:app_membrane}

We consider the integrated design of a membrane gas separation process and the polymer membrane material~\cite{qi2000membrane}. This problem exemplifies the permeability-selectivity trade-off characterized by the Robeson upper bound~\cite{robeson1991correlation,robeson2008upper}, which establishes that polymeric membranes exhibit a fundamental inverse relationship between gas permeability and selectivity explained by free volume theory~\cite{freeman1999basis}.

The white-box model follows the solution-diffusion transport mechanism~\cite{wijmans1995solution}, where the permeate flux $J_A$ is proportional to permeability, pressure driving force, and inversely proportional to membrane thickness $\delta$. The membrane area $A_m$, thickness $\delta$, and transmembrane pressure difference $\Delta p$ are the white-box decision variables, with the objective of maximizing product recovery while minimizing capital (area) and operating (compression) costs.

The black-box model represents structure-property relationships from molecular simulations or machine learning predictions~\cite{barnett2020designing} that map polymer descriptors---fractional free volume (FFV) and chain spacing $d_{\text{spacing}}$---to transport properties. The Robeson trade-off is encoded in the black-box functions: increasing FFV enhances permeability but reduces selectivity, while chain spacing affects both the upper bound position and mechanical integrity. The mechanical strength constraint reflects the requirement that the membrane must withstand the transmembrane pressure difference during module operation~\cite{park2017maximizing}.

The problem can be formulated as:
\begin{align} \label{eq:Membrane}
    \min\limits_{A_m, \delta, \Delta p, \text{FFV}, d_{\text{spacing}}} &~~ -F_A \cdot y_2 + 0.1 \cdot A_m + 10 \cdot \Delta p \\\notag
    \text{s.t.} &~~ J_A = \frac{y_1 \cdot \Delta p}{\delta} \cdot 10^{-10}, \\\notag
    &~~ F_A = J_A \cdot A_m, \\\notag
    &~~ g_1(x, y) := 10^{-6} - J_A \leq 0, \\\notag
    &~~ g_2(x, y) := \frac{\Delta p \cdot A_m}{y_3} - 100 \leq 0, \\\notag
    &~~ y_1 = f^\BB_1(\text{FFV}, d_{\text{spacing}}) := 1000 \cdot \exp\left(\frac{\text{FFV} - 0.15}{0.05}\right) \cdot \left(1 + 0.1 \cdot d_{\text{spacing}}\right), \\\notag
    &~~ y_2 = f^\BB_2(\text{FFV}, d_{\text{spacing}}) := 50 \cdot \exp\left(-\frac{\text{FFV} - 0.15}{0.1}\right) \cdot \exp\left(-\frac{(d_{\text{spacing}} - 5)^2}{4}\right), \\\notag
    &~~ y_3 = f^\BB_3(\text{FFV}, d_{\text{spacing}}) := 100 \cdot (0.4 - \text{FFV}) \cdot (10 - d_{\text{spacing}}), \\\notag
    &~~ A_m \in [10, 1000], \quad \delta \in [0.1, 10], \quad \Delta p \in [1, 20], \\\notag
    &~~ \text{FFV} \in [0.1, 0.3], \quad d_{\text{spacing}} \in [3, 8], \\\notag
    &~~ x = [x^{\WB}, x^{\BB}], \quad x^{\WB} = [A_m, \delta, \Delta p], \quad x^{\BB} = [\text{FFV}, d_{\text{spacing}}],
\end{align}
Refer to the parameter definitions in Table~\ref{tab:membrane_vars}.

\begin{table}[htbp]
\centering
\caption{Variables, physical meaning, bounds, and units (Membrane Separation)}
\label{tab:membrane_vars}
\begin{tabular}{@{}llll@{}}
\toprule
\textbf{Symbol} & \textbf{Description} & \textbf{Range / Value} & \textbf{Units} \\
\midrule
\multicolumn{4}{l}{White-box decision variables} \\
\midrule
$A_m$ & Membrane area & [10, 1000] & m$^2$ \\
$\delta$ & Membrane thickness & [0.1, 10] & $\mu$m \\
$\Delta p$ & Transmembrane pressure difference & [1, 20] & bar \\
\midrule
\multicolumn{4}{l}{Black-box decision variables} \\
\midrule
FFV & Fractional free volume & [0.1, 0.3] & --- \\
$d_{\text{spacing}}$ & Polymer chain spacing & [3, 8] & \AA \\
\midrule
\multicolumn{4}{l}{Black-box outputs} \\
\midrule
$y_1 = P_A$ & Permeability of component A & --- & Barrer \\
$y_2 = \alpha_{A/B}$ & Selectivity (A over B) & --- & --- \\
$y_3 = \sigma$ & Mechanical strength proxy & --- & MPa \\
\bottomrule
\end{tabular}
\end{table}

The white-box flux equation $J_A = y_1 \cdot \Delta p / \delta$ follows directly from the solution-diffusion model~\cite{wijmans1995solution}, where permeate flux is proportional to permeability and pressure driving force, and inversely proportional to membrane thickness. The factor $10^{-10}$ converts from Barrer (the conventional permeability unit) to SI flux units.

The black-box functions are constructed to capture established structure-property relationships from polymer physics. The permeability function $y_1$ follows the exponential free volume dependence derived by Cohen and Turnbull~\cite{cohen1959molecular} and extended by Fujita~\cite{fujita1961diffusion}, who showed that diffusivity (and hence permeability) scales as $P \propto \exp(B \cdot \text{FFV})$ where $B$ is a polymer-penetrant parameter. This exponential form arises because molecular transport occurs through transient gaps in the polymer matrix, with the probability of sufficiently large gaps increasing exponentially with free volume fraction~\cite{matteucci2006transport}. The chain spacing term provides a secondary contribution reflecting the effect of interchain distance on diffusion pathways~\cite{lee1980selection}.

The selectivity function $y_2$ encodes the Robeson upper bound trade-off~\cite{robeson1991correlation,robeson2008upper}: as FFV increases, permeability rises but selectivity decreases because larger free volume elements become less size-discriminating. Freeman~\cite{freeman1999basis} provided a theoretical basis showing that the slope of the permeability-selectivity trade-off depends only on penetrant size ratios. The Gaussian dependence on $d_{\text{spacing}}$ reflects an optimal chain spacing for size-sieving selectivity~\cite{park2017maximizing}.

The mechanical strength proxy $y_3$ captures the physical principle that increased free volume and chain spacing reduce polymer density and chain entanglement, thereby weakening mechanical integrity~\cite{bondi1964van}. While the linear functional form is a simplification, it correctly represents the inverse relationship between transport-enhancing microstructure and mechanical robustness that constrains practical membrane design~\cite{chiwaye2025superstructure}.

The constraint $g_1$ ensures minimum flux $J_A \geq 10^{-6}$ mol/(m$^2\cdot$s) for practical separation rates, and $g_2$ enforces mechanical stability by limiting the pressure-area product relative to membrane strength.

The global optimum for parameters in Table~\ref{tab:membrane_vars} occurs at $x^{\WB*} = [A_m, \delta, \Delta p] = [10, 0.1, 1]$ and $x^{\BB*} = [\text{FFV}, d_{\text{spacing}}] = [0.3, 5.13]$, yielding $J^* = 10.997$. The optimizer drives all white-box variables to their lower bounds---minimizing membrane area (capital cost) and pressure difference (operating cost) while maximizing flux through minimum thickness---and selects maximum FFV for permeability while tuning chain spacing to optimize the selectivity-permeability trade-off near the Gaussian peak at $d_{\text{spacing}} \approx 5$~\AA.

\subsection{Williams-Otto Process}\label{subsec:app_williams_otto}

We consider the Williams-Otto process~\cite{williams1960,schmid2020}, a benchmark chemical process consisting of a continuously stirred tank reactor (CSTR) with three reactions:
\begin{align*}
    \text{A} + \text{B} &\to \text{C}, &
    \text{C} + \text{B} &\to \text{P} + \text{E}, &
    \text{P} + \text{C} &\to \text{G},
\end{align*}
where P is the desired product and G is waste. We formulate this as a multi-scale grey-box problem where catalyst properties (black-box from DFT calculations) affect kinetic selectivity, while the reactor mass balances at steady state are known white-box equations. The problem can be formulated as:
\begin{align} \label{eq:WilliamsOtto}
    \max\limits_{T, F_{fB}, \mu, \varepsilon, \sigma_c} &~~ F_{pP} \\\notag
    \text{s.t.} &~~ k_i(T) = \frac{a_i}{\rho} \exp\left(-\frac{b_i}{T}\right), \quad i \in \{1, 2, 3\}, \\\notag
    &~~ r_1 = y_1 \cdot k_1(T) \cdot X_A X_B, \quad r_2 = y_2 \cdot k_2(T) \cdot X_B X_C, \quad r_3 = k_3(T) \cdot X_C X_P, \\\notag
    &~~ 0 = F_{fA} - \mu X_A - r_1 \cdot m, \\\notag
    &~~ 0 = F_{fB} - \mu X_B - r_1 \cdot m - r_2 \cdot m, \\\notag
    &~~ 0 = -\mu X_C + 2r_1 \cdot m - 2r_2 \cdot m - r_3 \cdot m, \\\notag
    &~~ 0 = -\mu X_E + 2r_2 \cdot m, \\\notag
    &~~ 0 = -\mu X_P + r_2 \cdot m - 0.5 r_3 \cdot m, \\\notag
    &~~ 0 = -\mu X_G + 1.5 r_3 \cdot m, \\\notag
    &~~ F_{pP} = \mu X_P - 0.1 \mu X_E, \quad F_{wG} = \mu X_G, \\\notag
    &~~ g_1(x, y) := F_{wG} - 1.0 \leq 0, \\\notag
    &~~ g_2(x, y) := 0.5 - X_A - X_B - X_C - X_E - X_P - X_G \leq 0, \\\notag
    &~~ y_1 = f^\BB_1(\varepsilon, \sigma_c) := \exp\left(-\frac{(\varepsilon - 15)^2}{50}\right) \cdot \left(1 + 0.1(\sigma_c - 5)\right), \\\notag
    &~~ y_2 = f^\BB_2(\varepsilon, \sigma_c) := 1.5 \cdot \exp\left(-\frac{(\varepsilon - 12)^2}{40}\right) \cdot \exp\left(-\frac{(\sigma_c - 6)^2}{8}\right), \\\notag
    &~~ T \in [400, 700]~^\circ\text{R}, \quad F_{fB} \in [5, 50]~\text{klb/h}, \quad \mu \in [50, 200]~\text{klb/h}, \\\notag
    &~~ \varepsilon \in [5, 25]~\text{kJ/mol}, \quad \sigma_c \in [3, 10]~\text{\AA}, \\\notag
    &~~ x = [x^{\WB}, x^{\BB}], \quad x^{\WB} = [T, F_{fB}, \mu], \quad x^{\BB} = [\varepsilon, \sigma_c],
\end{align}
where $X_i = m_i/m$ are mass fractions and $m$ is total reactor mass. Refer to the parameters in Table~\ref{tab:williams_otto_vars}.

\begin{table}[htbp]
\centering
\caption{Variables, physical meaning, bounds, and units (Williams-Otto Process)}
\label{tab:williams_otto_vars}
\begin{tabular}{@{}llll@{}}
\toprule
\textbf{Symbol} & \textbf{Description} & \textbf{Range / Value} & \textbf{Units} \\
\midrule
\multicolumn{4}{l}{White-box decision variables} \\
\midrule
$T$ & Reactor temperature & [400, 700] & $^\circ$R \\
$F_{fB}$ & Feed flow rate of B & [5, 50] & klb/h \\
$\mu$ & Total outlet flow rate & [50, 200] & klb/h \\
\midrule
\multicolumn{4}{l}{Black-box decision variables} \\
\midrule
$\varepsilon$ & Catalyst surface energy & [5, 25] & kJ/mol \\
$\sigma_c$ & Catalyst pore size & [3, 10] & \AA \\
\midrule
\multicolumn{4}{l}{Black-box outputs} \\
\midrule
$y_1$ & Activity factor for reaction 1 ($A + B \to C$) & --- & --- \\
$y_2$ & Activity factor for reaction 2 ($C + B \to P + E$) & --- & --- \\
\midrule
\multicolumn{4}{l}{Fixed parameters} \\
\midrule
$a_1$ & Pre-exponential factor (reaction 1) & $5.9755 \times 10^9$ & h$^{-1}$ \\
$a_2$ & Pre-exponential factor (reaction 2) & $2.5962 \times 10^{12}$ & h$^{-1}$ \\
$a_3$ & Pre-exponential factor (reaction 3) & $9.6283 \times 10^{15}$ & h$^{-1}$ \\
$b_1$ & Activation temperature (reaction 1) & 12000 & $^\circ$R \\
$b_2$ & Activation temperature (reaction 2) & 15000 & $^\circ$R \\
$b_3$ & Activation temperature (reaction 3) & 20000 & $^\circ$R \\
$\rho$ & Density & 50 & lb/ft$^3$ \\
$F_{fA}$ & Feed flow rate of A (fixed) & 10 & klb/h \\
\bottomrule
\end{tabular}
\end{table}

The black-box outputs $y_1$ and $y_2$ represent catalyst activity factors for reactions 1 and 2, modeled as volcano-type functions of catalyst surface energy $\varepsilon$ and pore size $\sigma_c$. The constraint $g_1$ limits waste production ($F_{wG} \leq 1.0$ klb/h) and $g_2$ ensures mass balance closure ($\sum_i X_i \geq 0.5$). The global maximum net product flow is $5.470$ klb/h with $x^{\WB\star} = [581.9, 50.0, 50.0]^\top$, $x^{\BB\star} = [12.6, 6.11]^\top$.

%% ============================================================================
%% APPENDIX B: ACQUISITION FUNCTIONS
%% ============================================================================
\section{Acquisition Function Details}\label{sec:acquisition_functions}

All BO experiments use Expected Improvement (EI) as the acquisition function. We also implemented four alternatives for completeness; the EI results are reported in the main text because EI is the most widely used baseline and enables direct comparison with prior work.

\subsection{Expected Improvement}\label{subsec:app_ei}

For minimization, Expected Improvement is defined as
\begin{equation}\label{eq:ei}
    \alpha_\text{EI}(x) = (\hat{y} - \mu(x) - \xi)\,\Phi(Z) + \sigma(x)\,\phi(Z), \quad Z = \frac{\hat{y} - \mu(x) - \xi}{\sigma(x)},
\end{equation}
where $\mu(x)$ and $\sigma(x)$ are the GP posterior mean and standard deviation, $\hat{y}$ is the best observed objective value, $\Phi$ and $\phi$ are the standard normal CDF and PDF, and $\xi \geq 0$ is an exploration parameter. When $\sigma(x) = 0$, we set $\alpha_\text{EI}(x) = 0$. For maximization, the improvement direction is reversed: $Z = (\mu(x) - \hat{y} - \xi) / \sigma(x)$.

The exploration parameter $\xi$ controls the trade-off between exploitation ($\xi = 0$, selecting points where the GP mean is low) and exploration ($\xi > 0$, favoring regions of high uncertainty even if the mean is not promising). We sweep $\xi \in \{0.001, 0.01, 0.05, 0.1, 0.2, 0.5, 1.0\}$ in all experiments.

\subsection{Probability of Improvement}\label{subsec:app_pi}

Probability of Improvement measures the probability that a candidate point improves over the current best:
\begin{equation}\label{eq:pi}
    \alpha_\text{PI}(x) = \Phi(Z), \quad Z = \frac{\hat{y} - \mu(x) - \xi}{\sigma(x)}.
\end{equation}
PI is a simpler alternative to EI that ignores the magnitude of improvement. It tends to exploit more aggressively than EI for the same $\xi$.

\subsection{Lower Confidence Bound}\label{subsec:app_lcb}

The Lower Confidence Bound (LCB) acquisition function for minimization is
\begin{equation}\label{eq:lcb}
    \alpha_\text{LCB}(x) = -\bigl(\mu(x) - \kappa\,\sigma(x)\bigr),
\end{equation}
where $\kappa > 0$ controls exploration. We negate LCB so that all acquisition functions follow the convention that higher values are better (i.e., the next point is selected by $\arg\max_x \alpha(x)$). For maximization, the Upper Confidence Bound $\alpha_\text{UCB}(x) = \mu(x) + \kappa\,\sigma(x)$ is used. We set $\kappa = 2.0$ by default.

\subsection{Modified Watson-Barnes 2}\label{subsec:app_mwb2}

The modified Watson-Barnes 2 (mWB2) acquisition function was proposed in the COBALT framework~\cite{paulson2022cobalt}:
\begin{equation}\label{eq:mwb2}
    \alpha_\text{mWB2}(x) = s_n \cdot \alpha_\text{EI}(x) - \mu(x), \quad s_n = -3\left(1 - \frac{n}{N}\right),
\end{equation}
where $n$ is the current iteration and $N$ is the total budget. The dynamic scaling factor $s_n$ starts large (emphasizing EI-driven exploration) and decreases toward zero (shifting weight to the GP mean for exploitation). For maximization, the sign of $\mu(x)$ is reversed.

\subsection{Thompson Sampling}\label{subsec:app_thompson}

Thompson Sampling draws a sample function $\tilde{f}$ from the GP posterior and selects the point that optimizes the sample:
\begin{equation}\label{eq:thompson}
    x_\text{next} = \arg\min\limits_{x \in \mathcal{X}} \tilde{f}(x), \quad \tilde{f} \sim \mathcal{GP}(\mu, k).
\end{equation}
In practice, we evaluate $\tilde{f}$ at a discrete set of candidate points by drawing from the multivariate normal $\mathcal{N}(\mu(X_\text{grid}), K(X_\text{grid}, X_\text{grid}))$. Exploration arises naturally from the randomness of the posterior draw rather than from an explicit exploration parameter.

\subsection{Acquisition Optimization}\label{subsec:app_acq_opt}

For all acquisition functions, we maximize $\alpha(x)$ over a random grid of 1{,}000 candidate points drawn uniformly from the search domain (either $\mathcal{X}^{\BB}$ for bilevel BO or $\mathcal{X}^{\WB} \times \mathcal{X}^{\BB}$ for black-box BO). This grid search approach avoids the gradient-based optimization of the acquisition function that can be problematic in low dimensions with multi-modal acquisition landscapes. The candidate with the highest acquisition value is selected as the next evaluation point.

%% ============================================================================
%% APPENDIX C: GP CONFIGURATION DETAILS
%% ============================================================================
\section{Gaussian Process Configuration}\label{sec:gp_details}

Both BO solvers use \texttt{scikit-learn}'s \texttt{GaussianProcessRegressor} with identical configuration. The kernel is a product of a constant kernel and an ARD RBF kernel, plus additive white noise:
\begin{equation}\label{eq:gp_kernel_full}
    k(x, x') = \sigma_f^2 \prod_{d=1}^{D} \exp\!\left(-\frac{(x_d - x_d')^2}{2\ell_d^2}\right) + \sigma_n^2\,\delta(x, x'),
\end{equation}
where $D$ is the input dimension ($n_{\BB}$ for bilevel BO, $n_{\WB}+n_{\BB}$ for black-box BO), $\sigma_f^2$ is the signal variance, $\ell_d$ is the length scale for dimension $d$ (automatic relevance determination), and $\sigma_n^2$ is the noise variance.

Table~\ref{tab:gp_hyperparams} summarizes the hyperparameter bounds and optimization settings.

\begin{table}[htbp]
\centering
\caption{GP hyperparameter configuration.}
\label{tab:gp_hyperparams}
\begin{tabular}{@{}lll@{}}
\toprule
\textbf{Parameter} & \textbf{Value / Bounds} & \textbf{Notes} \\
\midrule
Signal variance $\sigma_f^2$ & $[10^{-10},\; 10^3]$ & Initial value 1.0 \\
Length scales $\ell_d$ & $[10^{-10},\; 10]$ & One per input dimension (ARD) \\
Noise variance $\sigma_n^2$ & $[10^{-10},\; 10^{-1}]$ & Initial value $10^{-5}$ \\
Regularization $\alpha$ & $10^{-6}$ & Added to diagonal for numerical stability \\
Normalize $y$ & True & Observations standardized before fitting \\
Optimizer restarts & 5 & L-BFGS-B on log marginal likelihood \\
\bottomrule
\end{tabular}
\end{table}

The bilevel BO surrogate fits a GP over $n_{\BB} = 1$--2 input dimensions, while the black-box BO surrogate fits over $n_{\WB}+n_{\BB} = 2$--5 dimensions. Since GP fitting scales as $O(N^3)$ in the number of observations $N$ (which is the same for both methods) and the per-prediction cost is $O(N^2 D)$, the lower dimensionality of bilevel BO reduces prediction cost but does not change the dominant $O(N^3)$ training cost. In practice, with $N \leq 250$ and $D \leq 5$, GP fitting is fast for both methods (typically $<$0.1 seconds per iteration).

%% ============================================================================
%% APPENDIX D: BH BASELINE DETAILS
%% ============================================================================
\section{BH Baseline Details}\label{sec:nlp_details}

The black-box BH baseline solves the full problem~\eqref{eq:full_problem} directly using SciPy's \texttt{basinhopping} over the full variable space $[x^{\WB}, x^{\BB}]$ with an SLSQP local minimizer, a Metropolis temperature of $T=100$, full-range stepsize (uniform perturbation clipped to variable bounds), and 1{,}000 Basin-Hopping iterations.

\subsection{Constraint Handling}\label{subsec:bh_constraints}

Constraints are handled by the SLSQP local minimizer within Basin-Hopping, which enforces $g(x^{\WB},\allowbreak y,\allowbreak x^{\BB}) \leq 0$ as inequality constraints at each local minimization. After optimization, feasibility is verified: solutions satisfying all constraints to within $10^{-6}$ are considered feasible. If no feasible solution is found, the solver returns the solution with minimum constraint violation.

In contrast, both BO methods track feasibility explicitly. For black-box BO, constraints are handled via a penalty of $10^6$ added to the objective when any $g_i(x) > 0$. For bilevel BO, the inner BH solver checks constraint satisfaction and returns a penalty value when the inner problem is infeasible for a given $y$.

\subsection{Sample Efficiency Comparison}\label{subsec:bh_sample_efficiency}

Table~\ref{tab:sample_efficiency} compares the number of black-box evaluations across solvers. The BH baseline evaluates $f^{\BB}$ at every local minimization within Basin-Hopping, resulting in substantially more $f^{\BB}$ evaluations than either BO method.

\begin{table}[htbp]
\centering
\caption{Black-box evaluations ($f^{\BB}$ calls) per solver.}
\label{tab:sample_efficiency}
\begin{tabular}{@{}lrl@{}}
\toprule
\textbf{Solver} & \textbf{$f^{\BB}$ evaluations} & \textbf{Notes} \\
\midrule
Black-box BH & iteration-dependent & Multiple $f^{\BB}$ calls per BH iteration \\
Black-box BO & $n_\text{init} + 200$ & One $f^{\BB}$ per iteration \\
Bilevel BO (SLSQP) & $n_\text{init} + 200$ & One $f^{\BB}$ per iteration \\
Bilevel BO (BH) & $n_\text{init} + 200$ & One $f^{\BB}$ per iteration \\
\bottomrule
\end{tabular}
\end{table}

All three BO methods call $f^{\BB}$ exactly once per iteration, so the total sample budget is $n_\text{init} + N$ where $N=200$ is the number of BO iterations. With $n_\text{init} = 5$ (the smallest setting), this gives 205 evaluations. Both bilevel BO variants additionally solve an inner optimization problem at each iteration, but this uses only the white-box model $f^{\WB}$ (which is cheap) and does not require additional calls to $f^{\BB}$.

\subsection{Failure Modes}\label{subsec:bh_failures}

The BH baseline fails on Rastrigin, where multimodality in the full combined space challenges even global solvers: the BH baseline achieves a final regret of 3.38, worse than all BO methods. It also performs poorly on Distillation, where the tight constraint set and complex objective landscape limit convergence. On all other problems, the BH baseline approaches the global optimum but requires substantially more $f^{\BB}$ evaluations than the BO methods.

%% ============================================================================
%% APPENDIX: INNER-LOOP AND FULL-SPACE SOLVER BENCHMARKS
%% ============================================================================
\section{Inner-Loop and Full-Space Solver Benchmarks}\label{sec:solver_benchmarks}

Basin-Hopping (BH) was selected as the unified global solver for both the inner-loop white-box subproblem and the full-space NLP baseline. To justify this choice, we benchmarked five solvers across all 13 problems in both contexts: multi-start SLSQP with 50 Latin hypercube restarts (SLSQP-50), differential evolution (DE)~\cite{storn1997differential}, simplicial homology global optimization (SHGO)~\cite{endres2018simplicial}, dual annealing (DA)~\cite{xiang1997generalized}, and Basin-Hopping with 100 iterations (BH). We additionally ran a high-iteration Basin-Hopping variant with 1{,}000 iterations (BH-1000) to assess convergence with extended computation.

\subsection{Inner-Loop Solver Comparison}\label{subsec:inner_solver}

Table~\ref{tab:inner_regret} reports regret (distance from verified global optimum) for each solver on the inner white-box subproblem, where $x^{\BB}$ is fixed and only $x^{\WB}$ is optimized. All five solvers have access to analytical gradients of $J$ and $g$ with respect to $x^{\WB}$.

\begin{table}[htbp]
\centering
\caption{Inner-loop solver regret. A dash (---) indicates the solver returned an infeasible solution. All feasible regret values are near machine precision ($<10^{-7}$) except for DA, which fails to converge on six problems.}
\label{tab:inner_regret}
\footnotesize
\begin{tabular}{@{}lcccccc@{}}
\toprule
\textbf{Problem} & \textbf{SLSQP-50} & \textbf{DE} & \textbf{SHGO} & \textbf{DA} & \textbf{BH} & \textbf{BH-1000} \\
\midrule
Small-Feas.-Region-1 & 4.4e-16 & --- & --- & 0 & 4.4e-16 & 4.4e-16 \\
Small-Feas.-Region-2 & 0 & 4.4e-16 & 0 & 0 & 0 & 0 \\
Rastrigin & 1.8e-05 & 0 & 0 & 3.6e-05 & 1.4e-04 & 0 \\
Toy-Hydrology & 2.7e-08 & 2.7e-08 & 2.7e-08 & --- & 2.7e-08 & 2.7e-08 \\
Rosen-Suzuki & 1.6e-14 & 1.2e-14 & 0 & --- & 4.0e-15 & 3.9e-15 \\
CSTR & 0 & 8.1e-14 & 0 & 4.5e-14 & 1.5e-16 & 0 \\
Heat-Exchanger & 0 & 6.7e-14 & 0 & 0 & 0 & 0 \\
PSA & 0 & 3.0e-14 & 1.7e-16 & 4.3e-09 & 0 & 0 \\
Batch-Reactor & 2.0e-12 & 2.1e-12 & 2.0e-12 & --- & 1.9e-12 & 2.0e-12 \\
Distillation & 8.2e-09 & 8.4e-08 & --- & --- & 8.5e-08 & 8.4e-08 \\
Evaporator & 1.0e-11 & 1.0e-11 & 1.0e-11 & --- & 1.0e-11 & 1.0e-11 \\
Membrane & 0 & 6.7e-14 & 0 & 0 & 0 & 0 \\
Williams-Otto & 2.7e-11 & 5.0e-13 & 4.3e-13 & --- & 4.3e-13 & 4.3e-13 \\
\midrule
Feasible & 13/13 & 12/13 & 11/13 & 7/13 & 13/13 & 13/13 \\
\bottomrule
\end{tabular}
\end{table}

SLSQP-50 achieves the lowest regret on 7 of 13 problems, leveraging gradient information most directly. However, it is inapplicable to the derivative-free full-space problem (Section~\ref{subsec:full_solver}). SHGO and DA suffer from feasibility failures: SHGO returns infeasible solutions on 2 problems (Small-Feasible-Region-1 and Distillation), while DA fails on 6 problems (Toy-Hydrology, Rosen-Suzuki, Batch-Reactor, Distillation, Evaporator, and Williams-Otto). DE is reliable (12/13 feasible) but misses feasibility on one problem. BH and BH-1000 both achieve feasibility on all 13 problems. The regret differences among feasible solvers are negligible in practice---on the order of $10^{-14}$ to $10^{-8}$---well below any meaningful engineering threshold.

Table~\ref{tab:inner_time} reports the corresponding wall times. SLSQP-50 and DA are the fastest (median $<0.1$\,s), while SHGO is the slowest (median 3.1\,s), with several problems exceeding 20\,s. BH is fast at 100 iterations (median 0.20\,s) and remains practical at 1{,}000 iterations (median 0.45\,s), except for Williams-Otto (30.9\,s) and Small-Feasible-Region-1 (27.0\,s).

\begin{table}[htbp]
\centering
\caption{Inner-loop solver wall time (seconds).}
\label{tab:inner_time}
\footnotesize
\begin{tabular}{@{}lcccccc@{}}
\toprule
\textbf{Problem} & \textbf{SLSQP-50} & \textbf{DE} & \textbf{SHGO} & \textbf{DA} & \textbf{BH} & \textbf{BH-1000} \\
\midrule
Small-Feas.-Region-1 & 1.694 & 1.158 & 0.488 & 0.024 & 0.314 & 27.049 \\
Small-Feas.-Region-2 & 0.008 & 0.015 & 0.977 & 0.025 & 0.097 & 0.167 \\
Rastrigin & 0.012 & 0.058 & 25.235 & 0.050 & 0.121 & 0.200 \\
Toy-Hydrology & 0.028 & 0.034 & 0.601 & 0.024 & 0.204 & 0.505 \\
Rosen-Suzuki & 0.040 & 0.196 & 0.628 & 0.046 & 0.208 & 0.930 \\
CSTR & 0.023 & 0.284 & 13.349 & 0.079 & 0.247 & 0.607 \\
Heat-Exchanger & 0.009 & 0.246 & 3.054 & 0.070 & 0.104 & 0.170 \\
PSA & 0.030 & 0.295 & 1.627 & 0.072 & 0.236 & 0.448 \\
Batch-Reactor & 0.067 & 0.523 & 3.965 & 0.075 & 0.180 & 1.093 \\
Distillation & 0.016 & 0.433 & 21.461 & 0.068 & 0.168 & 0.375 \\
Evaporator & 0.018 & 0.382 & 0.866 & 0.073 & 0.207 & 0.336 \\
Membrane & 0.017 & 0.444 & 22.310 & 0.070 & 0.124 & 0.275 \\
Williams-Otto & 1.749 & 6.279 & 4.502 & 1.088 & 3.099 & 30.904 \\
\midrule
Median & 0.023 & 0.295 & 3.054 & 0.070 & 0.204 & 0.448 \\
\bottomrule
\end{tabular}
\end{table}

\subsection{Full-Space Solver Comparison}\label{subsec:full_solver}

Table~\ref{tab:full_regret} reports regret when each solver is applied to the full joint problem over $[x^{\WB}, x^{\BB}]$. SLSQP-50 is inapplicable in this context because the black-box function $f^{\BB}$ does not provide analytical gradients.

\begin{table}[htbp]
\centering
\caption{Full-space solver regret. A dash (---) indicates the solver returned an infeasible solution or is inapplicable. BH (100 iterations) fails catastrophically on Rastrigin due to multimodality in the combined space, but BH-1000 recovers.}
\label{tab:full_regret}
\footnotesize
\begin{tabular}{@{}lcccccc@{}}
\toprule
\textbf{Problem} & \textbf{SLSQP-50} & \textbf{DE} & \textbf{SHGO} & \textbf{DA} & \textbf{BH} & \textbf{BH-1000} \\
\midrule
Small-Feas.-Region-1 & --- & 7.3e-09 & 7.3e-09 & --- & 7.3e-09 & 7.3e-09 \\
Small-Feas.-Region-2 & --- & 3.8e-15 & 2.2e-16 & 3.6e-13 & 6.7e-16 & 0 \\
Rastrigin & --- & 0 & --- & 7.5e-04 & 9.9e+09 & 0 \\
Toy-Hydrology & --- & 2.7e-08 & --- & --- & 2.7e-08 & 2.7e-08 \\
Rosen-Suzuki & --- & 1.0e-13 & 1.3e-12 & --- & 3.1e-15 & 3.7e-15 \\
CSTR & --- & 1.5e-13 & --- & 2.0e-12 & 3.1e-16 & 0 \\
Heat-Exchanger & --- & 1.4e-13 & --- & --- & 0 & 0 \\
PSA & --- & 1.7e-13 & --- & --- & 0 & 0 \\
Batch-Reactor & --- & 2.2e-12 & 1.7e-12 & --- & 1.9e-12 & 1.9e-12 \\
Distillation & --- & 8.4e-08 & --- & --- & 8.6e-08 & 8.5e-08 \\
Evaporator & --- & 1.0e-11 & 9.8e-12 & --- & 9.8e-12 & 9.8e-12 \\
Membrane & --- & 2.5e-13 & --- & 1.8e-10 & 0 & 0 \\
Williams-Otto & --- & 5.9e-13 & 3.1e-13 & --- & 4.4e-13 & 4.4e-13 \\
\midrule
Feasible & 0/13 & 13/13 & 6/13 & 4/13 & 13/13 & 13/13 \\
\bottomrule
\end{tabular}
\end{table}

The full-space results exhibit sharper solver differentiation. SHGO fails on 7 of 13 problems, with 5 hitting the 60-second timeout on problems with higher dimensionality or tight constraints (Rastrigin, CSTR, Heat-Exchanger, PSA, Distillation, Membrane). DA fails on 9 of 13 problems and returns large-regret solutions even when feasible (e.g., $3.9 \times 10^6$ on PSA). DE is fully reliable (13/13 feasible) with uniformly low regret, but is the slowest among the feasible solvers.

BH achieves feasibility on all 13 problems with near-zero regret on 12 of them. The exception is Rastrigin, where BH at 100 iterations converges to a poor local optimum with a regret of $9.9 \times 10^9$. This failure is resolved by increasing to 1{,}000 iterations (BH-1000), which achieves zero regret. This is the only problem where BH requires more than 100 iterations to find the global optimum in the full space.

\subsection{Solver Selection Rationale}\label{subsec:solver_rationale}

Table~\ref{tab:solver_summary} summarizes the aggregate performance. BH is the only solver that achieves feasibility on all 13 problems in both contexts while also being applicable to derivative-free problems. Although SLSQP-50 wins on 7 of 13 inner-loop problems, it cannot be used for the full-space baseline. DE is the next most reliable solver (13/13 feasible in both contexts) but is slower (median 0.97\,s vs.\ 0.39\,s for BH in the full space). SHGO and DA are unsuitable as general-purpose solvers due to frequent feasibility failures.

\begin{table}[htbp]
\centering
\caption{Aggregate solver statistics across all 13 benchmark problems. Feasibility counts and median regret are computed over feasible solutions only. BH provides the best combination of universal feasibility, near-exact optimality, and broad applicability.}
\label{tab:solver_summary}
\footnotesize
\begin{tabular}{@{}lccccc@{}}
\toprule
& \multicolumn{2}{c}{\textbf{Inner-loop}} & \multicolumn{2}{c}{\textbf{Full-space}} & \\
\cmidrule(lr){2-3} \cmidrule(lr){4-5}
\textbf{Solver} & \textbf{Feasible} & \textbf{Med.\ regret} & \textbf{Feasible} & \textbf{Med.\ regret} & \textbf{Grad.-free} \\
\midrule
SLSQP-50 & 13/13 & 1.6e-14 & --- & --- & No \\
DE & 12/13 & 7.4e-14 & 13/13 & 2.5e-13 & Yes \\
SHGO & 11/13 & 0 & 6/13 & 1.5e-12 & Yes \\
DA & 7/13 & 0 & 4/13 & 9.2e-11 & Yes \\
BH & 13/13 & 4.0e-15 & 13/13 & 4.4e-13 & Yes \\
BH-1000 & 13/13 & 4.4e-16 & 13/13 & 3.7e-15 & Yes \\
\bottomrule
\end{tabular}
\end{table}

The regret gap between BH and the best solver on any given problem is negligible---typically $10^{-14}$ to $10^{-8}$---confirming that the choice of BH does not compromise solution quality. Because BH wraps an SLSQP local minimizer inside a global exploration loop, it inherits gradient-based precision on the inner problem while providing the global search capability needed for the full-space baseline. This combination of universal feasibility, near-exact optimality, and applicability to both derivative-available and derivative-free contexts makes BH the most defensible single-solver choice for the benchmark suite.

%% ============================================================================
%% APPENDIX E: GLOBAL OPTIMUM VERIFICATION
%% ============================================================================
\section{Global Optimum Verification}\label{sec:optimum_verification}

Each problem's global optimum $J^\star$ was verified through a multi-stage process designed to provide high confidence that the reported optima are correct.

\subsection{Verification Procedure}\label{subsec:verification_procedure}

We used two complementary approaches, each repeated across 10 independent random seeds. First, we ran Basin-Hopping (BH) over the full variable space $[x^{\WB}, x^{\BB}]$ with an SLSQP local minimizer, a Metropolis temperature of $T=100$, full-range stepsize, and 2{,}000 BH iterations. Constraints were enforced by the SLSQP local minimizer. Second, we ran multi-start SLSQP with 500 Latin hypercube starting points per seed.

For each problem, we compared the best feasible solutions across all 20 runs (10 BH + 10 multi-start NLP). All 13 problems achieved agreement to within $2 \times 10^{-4}$ relative error, confirming the reported optima. For constrained problems, we verified that the optimal solution satisfies all inequality constraints $g_i(x^{\WB\star}, y^\star) \leq 0$ to within a tolerance of $10^{-6}$.

\subsection{Verified Optima}\label{subsec:verified_optima}

Table~\ref{tab:verified_optima} reports the verified global optimum for each problem, along with the optimal decision variables.

\begin{table}[htbp]
\centering
\caption{Verified global optima for all 13 benchmark problems. Optima verified via Basin-Hopping (SLSQP local minimizer, $T=100$, 2000 iterations, 10 seeds) and multi-start SLSQP (500 Latin hypercube starts, 10 seeds).}
\label{tab:verified_optima}
\begin{tabular}{@{}lrll@{}}
\toprule
\textbf{Problem} & $J^\star$ & $x^{\WB\star}$ & $x^{\BB\star}$ \\
\midrule
Small-Feasible-Region 1 & 0.2532 & $[4.712]$ & $[1.253]$ \\
Small-Feasible-Region 2 & $-2.000$ & $[4.712]$ & $[5.500]$ \\
Rastrigin & 0.0000 & $[0, 0]$ & $[0]$ \\
Toy-Hydrology & 0.5998 & $[0.405]$ & $[0.195]$ \\
Rosen-Suzuki & $-44.00$ & $[0, 1]$ & $[2, -1]$ \\
\midrule
CSTR & 92.32\% & $[647.3,\; 20.0,\; 5.0]$ & $[-1.047,\; 0.200]$ \\
Heat-Exchanger & 1000.0 & $[1.0,\; 0.1,\; 19.3]$ & $[50.0,\; 0.0]$ \\
PSA & $-0.6433$ & $[3.97,\; 0.1,\; 10.0]$ & $[6.04,\; 19.6]$ \\
Batch-Reactor & $-1.749$ & $[500,\; 4.39,\; 5.0]$ & $[-1.01,\; 2.00]$ \\
Distillation & 11{,}758 & $[1.0,\; 11.0,\; 0.5]$ & $[19.8,\; 9.99]$ \\
Evaporator & $-1.959$ & $[4.09,\; 19.1,\; 5.0]$ & $[120.5,\; 2.0]$ \\
Membrane & 11.00 & $[10.0,\; 0.1,\; 1.0]$ & $[0.30,\; 5.13]$ \\
Williams-Otto & 5.470 & $[581.9,\; 50.0,\; 50.0]$ & $[12.6,\; 6.11]$ \\
\bottomrule
\end{tabular}
\end{table}

%% ============================================================================
%% APPENDIX F: PROOF OF PROPOSITION 1
%% ============================================================================
\section{Proof of Proposition~\ref{prop:dimensionality}}\label{sec:proof_dimensionality}

\begin{proof}
We show that $\inf(\text{bilevel}) = \inf(\text{original})$ by constructing mappings in both directions.

\emph{(i) Any feasible point of~\eqref{eq:full_problem} maps to a feasible outer-level point with equal or worse objective.}
Let $(x^{\WB}, x^{\BB})$ be feasible for~\eqref{eq:full_problem} with $y = f^{\BB}(x^{\BB})$. By Assumption~\ref{ass:separability}(ii), the inner problem~\eqref{eq:inner_obj}--\eqref{eq:inner_bounds} at this $x^{\BB}$ and $y$ has a global solution $x^{\WB\star}$ satisfying $J(x^{\WB\star},\allowbreak y,\allowbreak x^{\BB}) \leq J(x^{\WB},\allowbreak y,\allowbreak x^{\BB})$. Thus $x^{\BB}$ is outer-feasible with objective value no worse than the original point.

\emph{(ii) Any outer-level feasible point maps back to a feasible point of~\eqref{eq:full_problem} with equal objective.}
Let $\tilde{x}^{\BB}$ be outer-feasible with $\tilde{y} = f^{\BB}(\tilde{x}^{\BB})$ and inner solution $\tilde{x}^{\WB\star}$. By construction, $\tilde{x}^{\WB\star}$ satisfies $f^{\WB}(\tilde{x}^{\WB\star},\allowbreak \tilde{y}) = 0$, $g(\tilde{x}^{\WB\star},\allowbreak \tilde{y},\allowbreak \tilde{x}^{\BB}) \leq 0$, and $\tilde{x}^{\WB\star} \in \mathcal{X}^{\WB}$. Together with $\tilde{x}^{\BB} \in \mathcal{X}^{\BB}$, the pair $(\tilde{x}^{\WB\star}, \tilde{x}^{\BB})$ is feasible for~\eqref{eq:full_problem} with objective $J(\tilde{x}^{\WB\star}, \tilde{y}, \tilde{x}^{\BB})$.

\emph{(iii) Infeasibility.}
If the inner problem is infeasible for some $\tilde{y} \in \text{range}(f^{\BB})$---i.e., no $x^{\WB} \in \mathcal{X}^{\WB}$ satisfies $f^{\WB}(x^{\WB},\allowbreak \tilde{y}) = 0$ and $g(x^{\WB},\allowbreak \tilde{y},\allowbreak \tilde{x}^{\BB}) \leq 0$---then no feasible point of~\eqref{eq:full_problem} exists at this $\tilde{x}^{\BB}$ either, so excluding it from the outer search does not lose optimality.

From (i) and (ii), the feasible sets project onto each other with preserved objectives, so $\inf(\text{bilevel}) = \inf(\text{original})$. To see how a global optimum of~\eqref{eq:full_problem} appears in~\eqref{eq:bilevel}: let $(x^{\WB\star\star}, x^{\BB\star\star})$ be a global minimizer of~\eqref{eq:full_problem} with $y^{\star\star} = f^{\BB}(x^{\BB\star\star})$. By step (i), $x^{\BB\star\star}$ is outer-feasible in~\eqref{eq:bilevel}, and the inner problem at $x^{\BB\star\star}$ returns some $\hat{x}^{\WB}$ with $J(\hat{x}^{\WB},\allowbreak y^{\star\star},\allowbreak x^{\BB\star\star}) \leq J(x^{\WB\star\star},\allowbreak y^{\star\star},\allowbreak x^{\BB\star\star})$. Since $(x^{\WB\star\star}, x^{\BB\star\star})$ is already globally optimal for~\eqref{eq:full_problem}, the inequality must hold with equality: $\hat{x}^{\WB} = x^{\WB\star\star}$ (up to alternative optima with the same objective value). Thus the bilevel formulation recovers both the optimal $x^{\BB\star\star}$ at the outer level and the optimal $x^{\WB\star\star}$ at the inner level.

Assumption~\ref{ass:separability}(ii) is used in step (i) to ensure the inner global optimum is attained, and in step (ii) to ensure the inner solution is globally optimal (not merely a local minimum). The surrogate in Algorithm~\ref{alg:bilevel_bo} models the mapping~\eqref{eq:value_function}, which is a function of $n_{\BB}$ variables.
\end{proof}

%% ============================================================================
%% APPENDIX G: IMPLEMENTATION DETAILS
%% ============================================================================
\section{Implementation Details}\label{sec:implementation_details}

\emph{Infeasibility handling.} If the inner problem is infeasible for some $y$---because the black-box outputs lead to a white-box problem with no feasible solution---we return a large penalty value. The acquisition function naturally steers away from these regions without requiring explicit constraint modeling at the outer level.

\emph{Modularity.} The inner optimizer can be swapped (Basin-Hopping, SLSQP, IPOPT, BARON) without changing the outer BO framework. Similarly, any acquisition function (EI, PI, LCB) can be used in the outer loop.

%% ============================================================================
%% APPENDIX H: CONVERGENCE SPEED
%% ============================================================================
\section{Convergence Speed Details}\label{sec:convergence_speed_details}

\begin{table}[htbp]
\centering
\caption{Median iterations to reach 1\% of initial regret for Bi-BO (SLSQP) ($n_\text{init}=50$, best $\xi$, 10 repetitions). ``$>$200'' means the target was not reached within 200 iterations. Speedup values marked with $\geq$ are lower bounds (BB-BO did not converge).}
\label{tab:convergence_speed}
\begin{tabular}{@{}lrrr@{}}
\toprule
\textbf{Problem} & \textbf{BB-BO} & \textbf{Bi-BO (SL)} & \textbf{Speedup} \\
\midrule
Small-Feasible-Region   & $>$200 & $>$200 & ---       \\
Small-Feasible-Region-2 & $>$200 & 50     & $\geq$4.0$\times$  \\
Rastrigin               & $>$200 & 55     & $\geq$3.6$\times$  \\
Toy-Hydrology           & $>$200 & 59     & $\geq$3.4$\times$  \\
Rosen-Suzuki            & $>$200 & 134    & $\geq$1.5$\times$  \\
\midrule
CSTR                    & 87     & 50     & 1.8$\times$  \\
Heat-Exchanger          & 151    & 7      & 20.1$\times$ \\
PSA                     & $>$200 & 134    & $\geq$1.5$\times$  \\
Batch-Reactor           & $>$200 & $>$200 & ---       \\
Distillation            & $>$200 & 50     & $\geq$4.0$\times$  \\
Evaporator              & $>$200 & $>$200 & ---       \\
Membrane                & $>$200 & $>$200 & ---       \\
Williams-Otto           & $>$200 & 50     & $\geq$4.0$\times$  \\
\bottomrule
\end{tabular}
\end{table}

\begin{table}[htbp]
\centering
\caption{Post-initialization BO iterations required for Bi-BO (SLSQP) to reach BB-BO's median final regret, and vice versa ($n_\text{init}=50$, 200-iteration BO budget, 10 repetitions, best $\xi$ per problem per method). An iteration count of $0$ means the target is reached at the end of the initial sampling phase. ``$>$200'' indicates the target is never reached within the budget in the majority of repetitions. Bi-BO's NLP-assisted initial samples already dominate BB-BO's full-budget final on 11 of 13 problems.}
\label{tab:convergence_iterations}
\resizebox{\textwidth}{!}{%
\begin{tabular}{@{}lrrrrr@{}}
\toprule
\textbf{Problem} & \textbf{BB-BO final} & \textbf{Bi-BO final} & \textbf{Bi-BO iters $\to$ BB-BO} & \textbf{Iters saved} & \textbf{BB-BO iters $\to$ Bi-BO} \\
\midrule
Small-Feasible-Region   & 4.77$\times 10^{-3}$ & 4.36$\times 10^{-4}$ & 54 & $+146$ & $>$200 \\
Small-Feasible-Region-2 & 2.43$\times 10^{-2}$ & 3.09$\times 10^{-6}$ & 1  & $+199$ & $>$200 \\
Rastrigin               & 1.85                 & 2.65$\times 10^{-3}$ & 0  & $+200$ & $>$200 \\
Toy-Hydrology           & 1.14$\times 10^{-2}$ & 1.05$\times 10^{-6}$ & 0  & $+200$ & $>$200 \\
Rosen-Suzuki            & 4.88                 & 2.43$\times 10^{-2}$ & 0  & $+200$ & $>$200 \\
CSTR                    & 3.38$\times 10^{-1}$ & 1.87$\times 10^{-3}$ & 1  & $+199$ & $>$200 \\
Heat-Exchanger          & 3.71$\times 10^{1}$  & 3.64$\times 10^{-7}$ & 0  & $+200$ & $>$200 \\
PSA                     & 1.04$\times 10^{-1}$ & 2.99$\times 10^{-3}$ & 0  & $+200$ & $>$200 \\
Batch-Reactor           & 2.89$\times 10^{-1}$ & 4.63$\times 10^{-4}$ & 0  & $+200$ & $>$200 \\
Distillation            & 3.62$\times 10^{3}$  & 3.37$\times 10^{-3}$ & 0  & $+200$ & $>$200 \\
Evaporator              & 3.84$\times 10^{-1}$ & 1.45$\times 10^{-2}$ & 1  & $+199$ & $>$200 \\
Membrane                & 1.24                 & 4.28$\times 10^{-5}$ & 0  & $+200$ & $>$200 \\
Williams-Otto           & 8.25$\times 10^{-1}$ & 7.94$\times 10^{-6}$ & 0  & $+200$ & $>$200 \\
\bottomrule
\end{tabular}}
\end{table}

%% ============================================================================
%% APPENDIX I: HYPERPARAMETER SENSITIVITY DETAILS
%% ============================================================================
\section{Hyperparameter Sensitivity Details}\label{sec:hyperparameter_details}

We sweep over two key BO hyperparameters: the number of initial samples $n_\text{init} \in \{5, 25, 50\}$ drawn from a Latin hypercube design, and the EI exploration parameter $\xi \in \{0.001, 0.01, 0.05, 0.1, 0.2, 0.5, 1.0\}$. Each configuration is repeated 10 times with independent seeds 42--51 to quantify variability. Every BO run uses 200 iterations after initialization. In total, the experiment comprises 8{,}450 independent optimization runs: $13 \times 3 \times 7 \times 10 \times 3 = 8{,}190$ BO runs (one black-box and two bilevel variants) plus $13 \times 10 = 130$ NLP runs and $13 \times 10 = 130$ full-space BH runs (one each per problem-seed combination, since NLP and BH results do not depend on $n_\text{init}$ or $\xi$).

\begin{table}[htbp]
\centering
\caption{Bi-BO (SLSQP) mean final regret across $n_\text{init}$ values (best $\xi$ per problem per $n_\text{init}$, 10 repetitions).}
\label{tab:ninit_sensitivity}
\begin{tabular}{@{}lrrr@{}}
\toprule
\textbf{Problem} & $n_\text{init}=5$ & $n_\text{init}=25$ & $n_\text{init}=50$ \\
\midrule
Small-Feasible-Region   & 0.0014 & 0.0011 & 0.0004 \\
Small-Feasible-Region-2 & 0.0000 & 0.0000 & 0.0000 \\
Rastrigin               & 0.0077 & 0.0063 & 0.0026 \\
Toy-Hydrology           & 0.0000 & 0.0000 & 0.0000 \\
Rosen-Suzuki            & 0.0078 & 0.0607 & 0.0243 \\
\midrule
CSTR                    & 0.0015 & 0.0019 & 0.0019 \\
Heat-Exchanger          & 0.0000 & 0.0000 & 0.0000 \\
PSA                     & 0.0013 & 0.0022 & 0.0030 \\
Batch-Reactor           & 0.0003 & 0.0003 & 0.0005 \\
Distillation            & 0.0081 & 0.0056 & 0.0034 \\
Evaporator              & 0.0015 & 0.0085 & 0.0145 \\
Membrane                & 0.0000 & 0.0001 & 0.0000 \\
Williams-Otto           & 0.0000 & 0.0000 & 0.0000 \\
\bottomrule
\end{tabular}
\end{table}

\begin{table}[htbp]
\centering
\caption{Best EI exploration parameter $\xi$ per problem for Bi-BO (SLSQP), selected by lowest mean final regret across all $n_\text{init}$ values. These are the $\xi$ values used to report results in Table~\ref{tab:final_regret}.}
\label{tab:best_xi}
\begin{tabular}{@{}lr@{}}
\toprule
\textbf{Problem} & \textbf{best $\xi$} \\
\midrule
Small-Feasible-Region   & 0.01  \\
Small-Feasible-Region-2 & 0.001 \\
Rastrigin               & 1.0   \\
Toy-Hydrology           & 0.001 \\
Rosen-Suzuki            & 0.05  \\
\midrule
CSTR                    & 0.001 \\
Heat-Exchanger          & 0.001 \\
PSA                     & 0.001 \\
Batch-Reactor           & 0.001 \\
Distillation            & 0.001 \\
Evaporator              & 0.5   \\
Membrane                & 1.0   \\
Williams-Otto           & 0.001 \\
\bottomrule
\end{tabular}
\end{table}

% xi sensitivity heatmap
\begin{figure}[htbp]
    \centering
    \includegraphics[width=\textwidth]{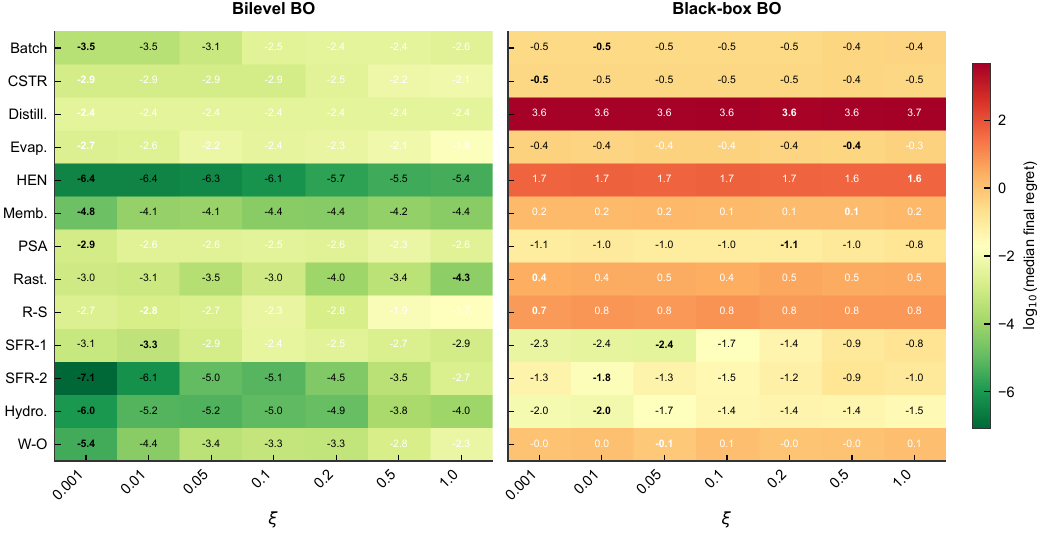}
    \caption{Sensitivity to the EI exploration parameter $\xi$. Left: bilevel BO; right: black-box BO. Rows are problems, columns are $\xi$ values, aggregated over all $n_\text{init}$. Cell values show log$_{10}$(median final regret); bold indicates best $\xi$ per problem. No single $\xi$ is universally optimal, but bilevel BO outperforms BB-BO regardless of $\xi$.}
    \label{fig:xi_heatmap}
\end{figure}

\end{appendices}

%% ============================================================================
%% BIBLIOGRAPHY (inlined per arXiv submission guidelines -- no external .bib file)
%% ============================================================================

\end{document}